\documentclass[review]{elsarticle}
\usepackage[letterpaper,top=2.5cm,bottom=2.5cm,left=2.6cm,right=2.6cm,marginparwidth=1.75cm]{geometry}
\usepackage{lineno,amsmath,amsfonts,amssymb,amsthm,mathrsfs}
\usepackage[colorlinks=true,linkcolor=blue]{hyperref}
\usepackage{algorithmic}
\usepackage{booktabs}
\usepackage[skip=0pt]{caption,subcaption}
\usepackage{float}
\usepackage{longtable}
\usepackage{setspace}
\biboptions{authoryear}
\newtheorem{theorem}{Theorem}

\newtheorem{corollary}{Corollary}

\begin{document}

\begin{frontmatter}
\title{Macroscopic Traffic Flow Modeling with Physics Regularized Gaussian Process: Generalized Formulations}

\author[1]{Yun Yuan}
\author[2]{Zhao Zhang}
\author[2]{Xianfeng Terry Yang\corref{cor1}*}
\ead{x.yang@utah.edu}
     
\address[1]{College of Transportation Engineering, Dalian Maritime University, Dalian, 116026, China}
\address[2]{Department of Civil \& Environmental Engineering, University of Utah, Salt Lake City, UT 84112, USA}
    \begin{abstract}
    Despite the success of classical traffic flow models and data-driven (e.g., Machine Learning - ML) approaches in traffic state estimation, those approaches either require great efforts in parameter calibrations or lack theoretical interpretations. 
    As a hybrid approach, Physics Regularized Gaussian Process (PRGP) can encode physics models, i.e., classical traffic flow models, into the Gaussian process (GP) architecture and so as to regularize the ML training process. However, the existing PRGP architecture requires the encoded physics model to be with continuous formulations, since the embedded augmented latent force model (LFM) uses a differential operator to process it. Such a strong assumption could significantly limit the applications of PRGP in a broader area. To address such an issue, this study proposes a generalized PRGP model, proves the existence of the regularization structure on a novel theoretical basis, and shows the applicability of a list of operators. Then, based on the derived approximate posterior objective function, an efficient alternating stochastic optimization algorithm is developed and proven. To show the effectiveness of the proposed model, this paper conducts empirical studies on a real-world dataset which is collected from a stretch of I-15 freeway, Utah. Results show the enhanced PRGP model can outperform the previous compatible methods, such as calibrated physics models and pure machine learning methods, in estimation accuracy and resistance to  data flaws.
    \end{abstract}


\end{frontmatter}

\section{Introduction}\label{sec:1}
In view of the steady increase of the number of vehicles and the occurrence of traffic congestion, traffic management represents an important alternative to improve the performance of traffic systems with limited efforts \citep{fountoulakis2017highway}.
As a precursive step of traffic management strategies, the full traffic state (i.e. flow, density, and speed) on highways should be estimated from the observed data (i.e. traffic counts, vehicle trajectories, etc.). However, in most cases, traffic state estimation (TSE) models can only utilize limited information from traffic detectors as inputs \citep{bekiaris2016highway}.

For example, traffic flow models were proposed based on continuum fluid approximation to describe the aggregated behavior of traffic.
Those models can generally be derived as partial differential equations (PDE) under ideal theoretical conditions, such as the first-order Lighthill-Whitham-Richards (LWR) model \citep{lighthill1955kinematic,richards1956shock}, the second-order Payne-Whitham (PW) model \citep{payne1971models,whitham1975linear}, and the second-order Aw-Rascle-Zhang (ARZ) model \citep{aw2000resurrection,zhang2002non}.
However, these models cannot be directly used to solve TSE. To address this issue, the previous studies discretized PDE formulations by the road segment and time period, such as the Godunov scheme \citep{lebacque1996godunov,daganzo1994cell}, the upwind scheme \citep{lebacque2007generic}, the Lax–Friedrichs scheme \citep{wong2002multi,gottlich2013numerical}, and the Lax–Wendroff scheme \citep{michalopoulos1993continuum}.
As a seminal work, \cite{papageorgiou1989macroscopic} discreted the PW model, METANET, and succeeded in reproducing complex traffic phenomena, and METANET and its reformulations have many successful applications in later studies.
To calibrate the traffic flow model in real-world applications, observations from stationary sensors (e.g., inductive loop, ultrasonic, radar, camera detectors) are usually leveraged and aggregated to average traffic flow and instant speed at a certain resolution.
However, their accuracy may be not reliable due to detection faults and uncertainties, such as frequent data missing and/or double counting of loop detectors \citep{chen2003bayesian}.
To account for such data uncertainties, researchers developed stochastic traffic flow models \citep{gazis1971line, szeto1972application, gazis2003kalman}, which were performed by adding Gaussian noise terms to the model expressions to capture those noises.
As a stochastic adaption to the base model, the stochastic METANET is enhanced by adding flow and speed errors in the formulation and its parameters are estimated by Extended Kalman filter (EKF) \citep{wang2005real}. Notably, Kalman filter (KF) and its extensions are well-known data assimilation methods, including unscented Kalman filter (UKF) \citep{mihaylova2006unscented}, ensemble Kalman filter (EnKF) \citep{work2008ensemble}, particle filter (PF) \citep{mihaylova2004particle}, etc.
However, \cite{jabari2012stochastic,jabari2013stochastic,seo2017traffic} pointed out that simply adding noise terms is theoretically flawed. 

With the advances in data collecting and processing technologies, data-driven methods have been developed dramatically in recent years. 
Data-driven methods does not require explicit theoretical assumptions, such as fundamental diagrams and conservation law \citep{smith2003exploring,chen2003detecting}. 
For example, machine learning (ML) models are prevailing in leveraging the voluminous data and capturing the stochasticity in TSE \citep{zhong2004estimation,ni2005markov,yin2012imputing,tang2015hybrid,tak2016data,li2013efficient,tan2014robust,tan2013tensor,duan2016efficient,polson2017deep,wu2018hybrid,polson2017bayesian,liang2018deep,xu2020ge}. However, due the data-driven nature, ML models are prone to data-induced errors. Lack of the high-quality data would unfortunately result in significant performance drops of ML models due to the detection system and random errors, communication failure, and storage malfunction. Hence, when the data contain unignorable outliers, pure ML estimation will be biased due to the misleading training data \citep{yuan2021macroscopic}. Although implementing a data screening and correction function before the ML training process could be helpful, in most cases, those incorrect data are not even able to be identified without further information \citep{lu2014algorithm}. 

Therefore, hybrid methodologies which fuse capability of the existing classical traffic flow models and pure ML models offer a new alternative to address the TSE challenges. Hybrid models also bridge the researches of classical traffic flow models and novel data-driven approaches. Among them, our pioneer work proposed the innovative Physics Regularized Machine Learning (PRML) model to leverage the well-investigated theoretical formulations, such as fundamental diagrams and conservation law, to overcome the flawed data challenge in ML theories \citep{yuan2021macroscopic}. Compared with physics (i.e., macroscopic traffic flow) models, the PRGP model can capture the uncertainties in estimation which beyond the capability of closed-form expressions and eliminate the efforts in calibrating model parameters. In comparison to pure ML models, the PRML is more resistant to the data noise/flaw as valuable knowledge from physics models can help to regularize the learning process. 

Fig.~\ref{fig:gpvsprgp} compares the concepts of pure Gaussian Process (GP) and Physics Regularized GP by depicting the observed traffic state, the GP estimated traffic state, and the PRGP estimated traffic state in three rectangles, where the blue square are for noisy observed states, pink squares represents biased observations, white squares represents the unobserved states, and green squares are for the estimated states. In the real-world cases, the raw data may be biased, noisy and missing due to system and communication failure, etc. Note that the flow, density and speed do not have physical meanings and are only separated isotropic dimensions in the pure GP model. To repair the data-bared flaw, the PRGP leverages the \emph{a priori} dynamics between the traffic state measures for improving the estimation accuracy and robustness.

\begin{figure}[H]
    \centering
    \includegraphics[width=0.8\textwidth]{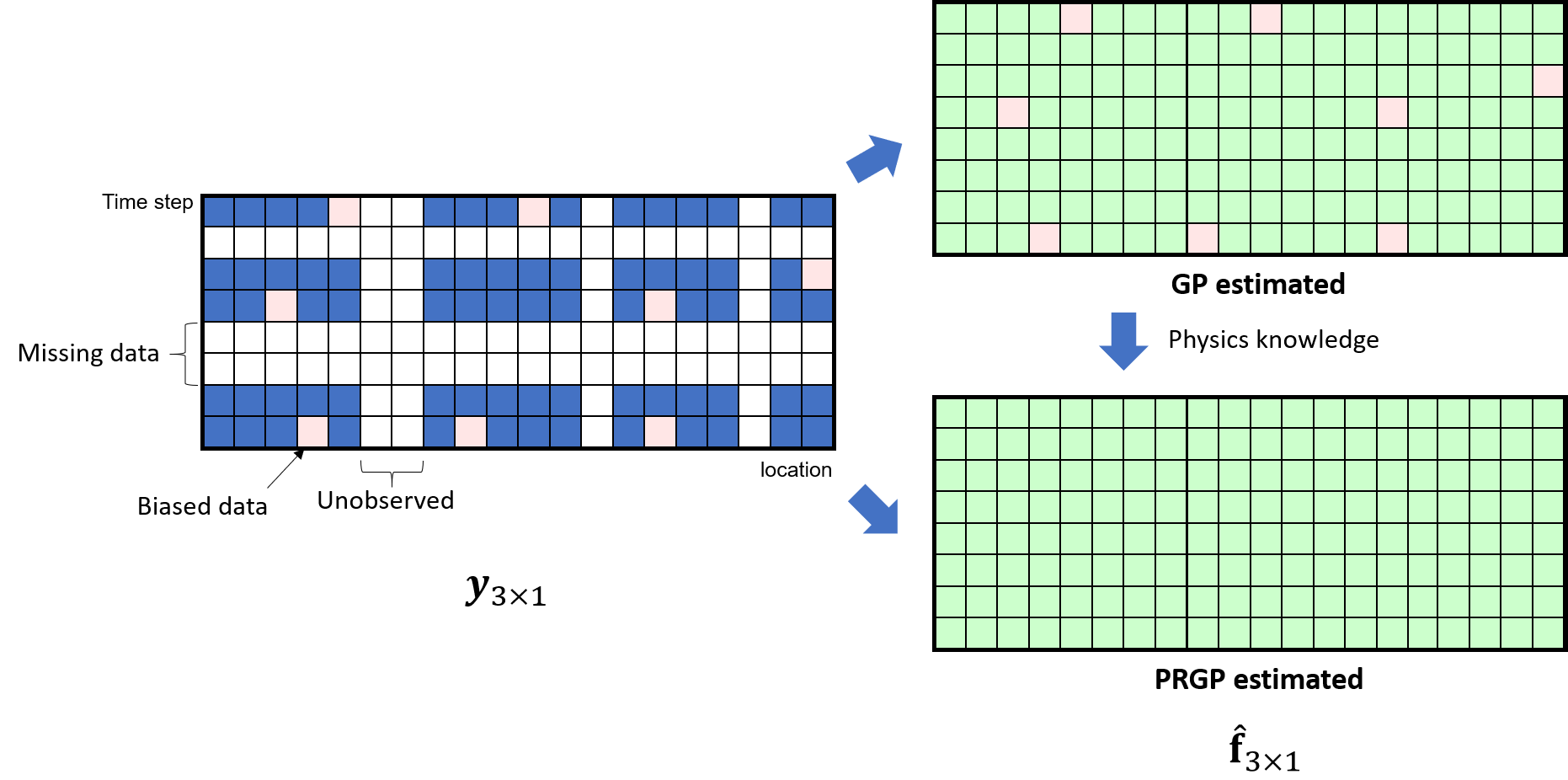}
    \caption{Conceptional comparison between PRGP and GP}
    \label{fig:gpvsprgp}
\end{figure}

However, relied on the augmented latent force model (LFM), the previous PRGP model can only employ PDEs, such as continuous traffic models, as the regularizer due to its theoretical basis. The applicability of PRGP on non-PDE equations, especially discretized traffic flow models, is unexplored. Although the current PRGP model is designed for physics models formulated in PDEs, it may also be applicable to non-PDE equations by similar encoding techniques. Hence, to investigate the applicability of the PRGP model in a broader application domain, this study aims to further advance this foundational theory by developing a new modeling method to encode non-PDE models into GP. Accordingly, this study also reformulates the evident lowerbound of the log-posterior to fix the compatibility of the PRGP model and the discretized traffic flow models. 

More specifically, this study contributes to the literature in the following aspects: \\
(a) To extend the capability of PRGP and remove the dependency on the augmented LFM, this paper rebases the theoretical basis by reformulating a generalized PRGP, proving the existence of physics based GP, and presenting the necessary condition of encoding the physics models in the PRGP model; \\
(b) To infer the generalized PRGP, an efficient alternating stochastic optimization algorithm is developed by deriving the objective function and proving the correctness of the Bayesian stochastic algorithm on the generalized PRGP; and \\
(c) This paper conducts the real-world case study to validate the capability and the robustness of the generalized PRGP with a discretized traffic model. 

The remainder of this paper is organized as follows.
Section~\ref{sec:2} shows the discrete TSE, GP and PRGP modeling.
In Section~\ref{sec:3}, the integrated GP and physics model equations are formulated for encoding physical knowledge into Bayesian statistics, and the posterior regularized inference algorithm are presented.
In Section~\ref{sec:4}, the case study on a real-world data from the interstate freeway I-15 is conducted to justify the proposed methods.
The conclusion section summarizes the critical findings and future research directions.

\section{Review of Related Models}\label{sec:2}
\subsection{Notations and Variable Definitions}
For the convenience of discussion, Table~\ref{tab:notation} summarizes key notations that have been used in the generalized PRGP model:

\begin{longtable}{p{5cm}p{10cm}}
\caption{List of key notations in this study}\label{tab:notation}\\
\toprule
Notation & Definition\\
\midrule
PRGP model notations\\
\midrule
$\mathcal{D}$ & the training data set;\\
$d,d^\prime$ & the dimensions of the input and output, respectively;\\
$a$
$f,\hat{f}$ & the (estimated) mapping from $\mathbf{x}$ to $\mathbf{y}$;\\
$\mathbf{f}$ & the function value of the mapping $f$;\\
$\hat{\mathbf{f}}$ & the estimated function value;\\
$g$& the right-hand side value of physical equations;\\
$\mathbf{g}$ & the vector of the right-hand side of physical equations;\\
$\mathbf{I}$ & the identity matrix;\\
$j,p$ & the index of the observation in the data set;\\
$K$ & the kernel function;\\
$\mathbf{K_f}$ & the kernel value matrix regarding $\mathbf{X}$;\\
$\mathbf{\hat{K}},\mathbf{K_g}$ & the kernel value matrix regarding inputs $\mathbf{Z}$;\\
$\mathbf{K}_*$ & the kernel value matrix regarding the new inputs $\mathbf{X}^*$;\\
$k$ & the index of the time step;\\
$\mathcal{L}$ & the objective function and the lowerbound of evidence lowerbound;\\
$\mathcal{N}$ & the vectorized Gaussian distribution;\\
$\mathbb{N}$ & the natural number set; \\
$n$ & the number of observations, in another word, the sample size;\\
$m$ & the number of pseudo observations;\\
$t$ & the index of the algorithm iteration;\\
$W$& the total number of physics equations;\\
$w$& the index of physics equations;\\
$\mathbf{X}$ & the data input vectors of size $n$;\\
$\mathbf{X}^*$ & the separated input vectors for estimation;\\
$\mathbf{x}$ & the model input vector, i.e. location, time;\\
$\mathbf{Y}$ & the data output vectors of size $n$;\\
$\mathbf{y}$ & the model output vector, i.e. flow, speed, density;\\
$\mathbf{Z}$ & the pseudo-observation input vector of size $m$;\\
$\mathbf{z}$ & the pseudo-observation input;\\
$\bar{\tau}$ & the isotropic Gaussian noise level;\\
$\mu,\sigma$ & the mean and standard deviation of the probability distribution;\\
$\eta_1, \eta_2, \eta_3, \ldots$ & kernel parameters;\\
$\mathbf{0}$ & the pseudo-observation output vector;\\
\midrule
METANET model notations\\
\midrule
$I$ & the number of the highway segment;\\
$i$ & the index of the highway segment;\\
$q_{i,k}$ & the total flow at the end of segment $i$;\\
$r_i$ & the inflow of vehicles at on-ramps;\\
$s_i$ & the outflow of vehicles at off-ramps;\\
$T$ & the time-discretization step;\\
$v_f$       &  the free-flow speed;\\
$v_{i,k}$ & the average speed at segment $i$;\\
$\alpha$    &  the exponent of the stationary speed equation;\\
$\beta_{i,k}$ & the departure rate;\\
$\Delta_i$ & the segment length at the segment;\\
$\nu,\delta,\tau,\kappa$    &  the model parameters;\\
$\rho_{i,k}$ & the density at the end of segment $i$;\\
$\rho_{cr}$ & the critical density;\\
$\lambda_i$ &  the number of lanes of segment $i$;\\
$\xi_{i,k}^q$ &  the zero-mean Gaussian white noise acting on the empirical flow equation;\\
$\xi_{i,k}^v$ &  the zero-mean Gaussian white noise acting on the empirical speed equation;\\
\midrule
Algorithm notations\\
\midrule
$\mathcal{D}$ & the dataset;\\
$\mathbf{K_f,K_g}$ & the kernel matrix of a specific input vector;\\
$\mathbf{k_x}$ & the kernel function of a specific input;\\
$\mathbf{K}_{**},\mathbf{k}_*$ &kernel matrix of new inputs;\\
$\mathcal{L}$ & the objective function, the sum of evidence lowerbound of posterior;\\
$\mathcal{L}_v,\mathcal{L}_q,\mathcal{L}_f,\mathcal{L}_g$ & partial terms of the objective function;\\
$s$ & the number of pseudo-input points;\\
$\mathbf{\bar{X}}$ & pseudo-input points;\\
$\bar{\mathbf{x}}_i,\bar{\mathbf{x}}_{i^\prime}$& a specific data point;\\
$\mathbf{\bar{Y}}$ & pseudo-outputs;\\
$\mathbf{\bar{f}}$ & pseudo-estimations;\\
$\mathbf{X^*},\mathbf{y}^*$ & new inputs and targets;\\
$\Lambda$ & the diagonal kernel vector;\\
$\gamma$ & the positive coefficient for the regularization effect;\\
$\mu_*,\sigma_*$ & mean and variance of new inputs and targets;\\ 
$\theta$ & the vector of all trainable kernel and model parameters;\\
$\theta^{(t)}$ & the value of parameter at the $t^{th}$ iteration;\\
$\phi$ & learning rate.\\
\bottomrule
\end{longtable}

\subsection{Second order traffic flow model and its stochastic extensions}
Discretizing partial differential equation is well investigated in the literature \citep{wang2022real}. In this section, we take an example for encoding a discretized traffic flow model in the PRGP modeling instead of enumerate existing discretized models exhaustively.
As an influential study in the literature, \cite{papageorgiou1989macroscopic} proposed a discrete macroscopic traffic flow model, METANET, which subdivided the highway sketch into $I$ segments and considered the density $\rho_{i,k}$ of highway segment $i=1,\ldots, I$ at time step $k$ to be the number of vehicles in the segment divided by the segment length $\Delta_i$. The dynamics of the density can be described by Eq.~\ref{eq:wp1}.

\begin{equation}
    \rho_{i,k+1}=\rho_{i,k}+\frac{T}{\Delta_i\lambda_i}[q_{i-1,k}-q_{i,k}+r_{i,k}-s_{i,k}]
    \label{eq:wp1}
\end{equation}
The departure flow is assumed to be a portion of the flow at the segment in Eq.~\ref{eq:wp2}. It is assumed that any unmeasured on-ramp and off-ramp are constant, or, effectively, slowly varying so that the ramp flow may be captured by a random walk.

\begin{equation}
    s_{i,k}=\beta_{i,k}\cdot q_{i-1,k}
    \label{eq:wp2}
\end{equation}
The dynamics of the speed can be described by Eq.~\ref{eq:wp3}.

\begin{equation}
    v_{i,k+1}=v_{i,k}+\frac{T}{\tau}[V(\rho_{i,k})-v_{i,k}]+\frac{T}{\Delta_i}v_{i,k}(v_{i-1,k}-v_{i,k})-\frac{\nu T}{\tau\Delta_i}\frac{\rho_{i+1,k}-\rho_{i,k}}{\rho_{i,k}+\kappa}-\frac{\delta T}{\Delta_i\lambda_i}\frac{r_{i,k}v_{i,k}}{\rho_{i,k}+\kappa}
    \label{eq:wp3}
\end{equation}
The exponential fundamental diagram is shown in Eqs.~\ref{eq:wp4}-\ref{eq:wp5}.

\begin{equation}
    V(\rho)=v_f\textit{exp}\Big[-\frac{1}{\alpha}(\frac{\rho}{\rho_{cr}})^\alpha\Big]
    \label{eq:wp4}
\end{equation}

\begin{equation}
    q_{i,k}=\rho_{i,k}v_{i,k}\lambda_{i}
    \label{eq:wp5}
\end{equation}
where Eqs.~\ref{eq:wp1}, \ref{eq:wp3}, \ref{eq:wp4}, \ref{eq:wp5} are the well-known conservation equation, dynamic speed equation, stationary speed equation, and continuity equation, respectively; $\tau, \nu, \delta, \kappa, v_f, \rho_{cr}, \alpha$ are positive model parameters which are given the same values for all segments, specifically, $v_f$ denotes the free-flow speed, $\rho_{cr}$ the critical density, and $\alpha$ the exponent of the stationary speed equation. 
Considering the limitation of the METANET model in representing real-world traffic fluctuations, \cite{wang2005real} added Gaussian error terms $\xi^v_{i,k}, \xi^q_{i,k}$ to the flow and speed equations (Eqs.~\ref{eq:wp3r}-\ref{eq:wp5r}) to capture the random errors of traffic detectors. 

\begin{equation}
    v_{i,k+1}=v_{i,k}+\frac{T}{\tau}[V(\rho_{i,k})-v_{i,k}]+\frac{T}{\Delta_i}v_{i,k}(v_{i-1,k}-v_{i,k})-\frac{\nu T}{\tau\Delta_i}\frac{\rho_{i+1,k}-\rho_{i,k}}{\rho_{i,k}+\kappa}-\frac{\delta T}{\Delta_i\lambda_i}\frac{r_{i,k}v_{i,k}}{\rho_{i,k}+\kappa}+\xi^v_{i,k}
    \label{eq:wp3r}
\end{equation}

\begin{equation}
    q_{i,k}=\rho_{i,k}v_{i,k}\lambda_{i} + \xi^q_{i,k}
    \label{eq:wp5r}
\end{equation}
where $\xi^v_{i,k}, \xi^q_{i,k}$ denote zero-mean Gaussian white noise acting on the empirical equations and the approximate speed and flow equations, respectively, to reflect the modeling inaccuracies. Then an EKF function is implemented to dynamically correct the model estimates based on detector measurements. 
Notably, despite the successful applications and extensions, the EKF-based model may possibly produce infeasible behaviors, such as negative speed and information propagating faster-than-vehicle speed. This is due to the fact that nonlinear functions of Gaussian noise typically produce non-Gaussian and non-zero mean random noises \citep{daganzo1995requiem,del1994reaction,hoogendoorn2001state,papageorgiou1998some}. In the meantime, the calibration of model parameters and EKF initial covariance matrix, which often requires tremendous efforts, plays a key role in affecting TSE accuracy. In this study, METANET and its extended version with EKF will both serve as benchmark models to evaluate the performance of the proposed PRGP model.  

\subsection{Review of Gaussian Process and Physics Regularizer}
This section reviews the key concept of the conventional Gaussian Process (GP) and its applications in the TSE problem, describes the modeling structure of the PRGP, and illustrates how to encode the physical knowledge (i.e. traffic flow model) into the GP model.

GP is a data-driven method for capturing the similarity between the system states, of which the core idea is to learn the kernel function (i.e. covariance) between variables and to predict (or estimate) the target by the linear combination of the training data \citep{rasmussen2003gaussian}. 
GP assumes that the Gaussian noise exists in the data $\mathbf{Y}$. Given the data $\mathbf{X},\mathbf{Y}$ and the new input $\mathbf{x}^*$, the noise-free function value $\mathbf{f}$ can be estimated based on Eq.~\ref{eq:pfxxy}, where the kernel $K$ is defined as the non-parametric smooth positive-definite covariance function with parameters $\eta_1,\eta_2,\eta_3,...$ \citep{bishop2006pattern}.  
\begin{equation}
    \label{eq:pfxxy}
    p(\mathbf{f}(\mathbf{x}^*)|\mathbf{x}^*,\mathbf{X},\mathbf{Y})=\mathcal{N}(\mu(\mathbf{x}^*),\sigma(\mathbf{x}^*))
\end{equation}
\begin{equation}
    \label{eq:mu_def}
    \mu(\mathbf{x}^*)=\mathbf{K}_*^\intercal (\mathbf{K}+\bar{\tau}^{-1} \mathbf{I})^{-1} \mathbf{Y}
\end{equation}
\begin{equation}
    \label{eq:nu_def}
    \sigma(\mathbf{x}^*)=K(\mathbf{x}^*,\mathbf{x}^*)-\mathbf{K}_*^\intercal(\mathbf{K}+\bar{\tau}^{-1} \mathbf{I})^{-1} \mathbf{K}_*
\end{equation}
\begin{equation}
    \label{eq:k_def}
    \mathbf{K}_*=
    \begin{bmatrix}
    K(\mathbf{x^*,x_1})&\ldots&K(\mathbf{x^*,x_n})
    \end{bmatrix}
    ^\intercal
\end{equation}

To apply GP in the TSE problem, the designed concept is illustrated in Fig.~\ref{fig:framework}, where the discrete traffic state estimation problem is described as taking the inputs $q,v$ from the stationary detectors at $0,\ldots,i-1,\ldots$ or probe vehicles to estimate the unobserved traffic state $q,v$ at the other locations. This model integrates the stochastic METANET model and GP, of which the key task is to learn the kernel functions of traffic flow $K^{(q)}$ and the kernel functions of traffic speed $K^{(v)}$. The kernel function is defined as the covariance of the values of traffic flow (or traffic speed) at two locations or two time intervals. Empirically, the formulations of the kernel functions can be selected to be same or different. 
The input $\mathbf{x}$ represents the index of the segment and the time step, the output $\mathbf{y}$ represents the corresponding vector of flow, density, speed. Leveraging the GP, we can predict the unobserved traffic states $\hat{\mathbf{f}}$ from the samples $\mathbf{(X,Y)}$.

\begin{figure}[H]
    \centering
    \includegraphics[width=0.8 \textwidth]{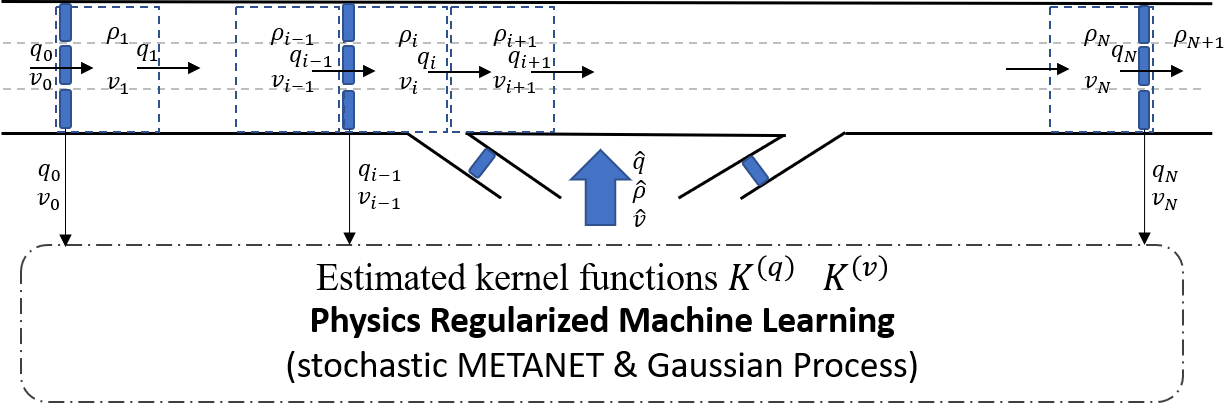}
    \caption{The proposed model for physics regularized Gaussian process learning}
    \label{fig:framework}
\end{figure}

However, it should be noted that GP is limited in addressing data quality issue and showing interpretability by physical meanings. This is also commonly recognized as a critical limitation of pure data-driven approaches and many ML models suffer from the same deficiency. To address this issue, \cite{wang2020physics} introduced the general Physics Regularized Machine Learning concept to extend the conventional GP to incorporate PDEs as the regularizer in the posterior inference algorithm. The physics model-based regularization is conducted by encoding the physics equations into GPs and adding the corresponding log-posterior into the inference objective function as a penalty term. The Latent Force Model (LFM) \citep{alvarez2013linear} is augmented to create a generative component for regularizing the original GP with a differential equation. The original LFM assumes the formulation of the PDE is given and the differential result is decomposed with the Green's function. The original LFM is solved by assigning a GP prior and the restrictive convolution operation. The augmented LFM is solved by conducting differentiation operation to obtain the latent force and regularizing it with another GP prior in a reversed direction. Using the augmented LFM, the differential equation is encoded into the so-called shadow GP. 
Despite the capability of encoding PDE into GP, the original PRGP model was developed with a single output variable, and was tested on one single-variable differentiable physics equations. Following the same line, our later study \citep{yuan2021macroscopic} extended the PRGP model to handle the multiple outputs and multiple physics equations simultaneously, and applied the PRGP to the TSE problem. To address this issue, the PRGP employs the valuable physical knowledge, from the classical traffic flow models, to regularize the training process for more robust performances. In the PRGP model, the physical knowledge (i.e. traffic flow models) is encoded into GPs, which captures both the stochasticity due to flawed/noisy data as well as the unobserved factors, such as missing on-ramp or off-ramp data. 
Given the differential operator $\Psi$ can be linear or nonlinear physics differential operator, the augmented LFM equation is formulated in Eq.~\ref{eq:nonlinearlfm} \citep{yuan2021macroscopic}. Augmented LFM is based on solving the PDE numerically since $\Psi$ is defined as a differential operator, where $g(\cdot)$ represents the unknown latent force functions, $f(\mathbf{x})$ is the function to be estimated from data $\mathcal{D}$. 
\begin{equation}
    \label{eq:nonlinearlfm}
    \Psi f(\mathbf{x}) = g(\mathbf{x})
\end{equation}

In the previous works, the PRGP model is developed from using the augmented LFM to solve the PDE in a data-driven framework. Despite the successful application in addressing data randomness and flaw, the previous PRGP model can only employ PDEs as the regularizer due to its theoretical basis. When the PDEs cannot be obtainable, the applicability of PRGP on the discretized traffic models is not proven. Whether this PDE-oriented method would work is questionable since discretized models are neither continuous nor differentiable. It is not determined that whether the non-PDE physics equations can be encoded in the PRGP model. Particularly, in the numerical experiment, the continuous PDE can encourage the smooth convergence of the algorithm. Thus, it is a challenge to enhance the PRGP with applicability of generalized traffic models. 

\section{Generalized Physics Regularized Gaussian Process}\label{sec:3}

\subsection{Model Development}
To fill the existing research gap, this paper rebases the theoretical basis of the PRGP model. By removing the augmented LFM, this study generalizes PRGP to encode the non-PDE physics equations, such as the discretized traffic flow models. The theory is developed in three steps: (a) proving the existence of physics based GP, which serves as the theoretical basis of establishing the PRGP model; (b) deriving the objective function of inferring the PRGP model, which shows the computational process of the inference algorithm; and (c) presenting the necessary condition of encoding the physics models in the PRGP model.

The physics equations are supposed to be in the canonical form of Eq.~\ref{eq:physics}, where $\Phi$ refers to the linear or nonlinear physics operator, $\mathbf{f(Z)}$ is the true output value. In the discretized model, the physics equations are converted into the desired function forms by moving terms to one side of equation and let the other side be zero. 
\begin{equation}
    \mathbf{\Phi[f(Z)]=0}
    \label{eq:physics}
\end{equation}

Considering the unobserved latent value and the random error, the physics equation is encoded into PRGP in form of Eq.~\ref{eq:gz}, where $\mathbf{g}$ is assumed to be a GP, $\mathbf{\hat{f}(Z)}$ is the estimated outputs upon the input $\mathbf{Z}$. When the data perfectly meets with the physics model function, the remaining error $\mathbf{g}$ is supposed to have zero mean and zero variance, which is consistent with Eq.~\ref{eq:physics}. 

\begin{equation}
    \mathbf{\Phi[\hat{f}(Z)]=g}
    \label{eq:gz}
\end{equation}

To establish the PRGP model, Theorem~\ref{thm:1} shows the existence of another GP by applying the physics model on the original GP.
\begin{theorem}\label{thm:1}
    Given $\mathbf{\Phi}[\cdot]$ is a physics model function of the output of the GP $\mathbf{f}$ of data $\mathcal{D}$, these exists a GP $\mathbf{g}$ satisfying the following equation. 
    \begin{equation}
        \label{eq:function}
        \Phi \mathbf{f}(\mathbf{x}) = \mathbf{g}(\mathbf{x})
    \end{equation}
\end{theorem}

\begin{proof}
    The idea of the proof is to apply the physical operator on the mean and variance expressions, the resultant expressions are in the form of mean and variance of another GP. It means the physical operator is applied on the kernel function.
    Given the original GP upon the observation data $\mathcal{D}=(\mathbf{X,Y})$, the mean of the estimation can be formulated in the following equation. 
    \begin{equation}
        \mu_{\mathbf{f}}(\mathbf{z}^*)=\mathbf{K}_{\mathbf{f}*}^\intercal (\mathbf{K}_{\mathbf{f}}+\bar{\tau}_{\mathbf{f}}^{-1} \mathbf{I})^{-1} \mathbf{Y}
    \end{equation}
    Since the mean is the point with maximum probability of the Gaussian distribution, it is also used as the estimation of the outputs $\mathbf{\hat{f}}$ regarding the pseudo-inputs $\mathbf{z}$, as shown in the following equation.
    \begin{equation}
        \mathbf{\hat{f}}(\mathbf{z}^*) = \mu_{\mathbf{f}}(\mathbf{z}^*)
    \end{equation}
    Similarly, the r.h.s. of Eq.~\ref{eq:function} is formulated in the following equation.
    \begin{equation}
        \mu_{\mathbf{g}}(\mathbf{z}^*)=\mathbf{K}_{\mathbf{g}*}^\intercal (\mathbf{K}_{\mathbf{g}}+\bar{\tau}_{\mathbf{g}}^{-1} \mathbf{I})^{-1} \mathbf{\hat{f}}(\mathbf{z}^*)
    \end{equation}
    By applying the physical operator $\Phi$, the following equation holds.
    \begin{equation}
        \mu_{\mathbf{g}}(\mathbf{z}^*) = \Phi \mathbf{\hat{f}}(\mathbf{z}^*)
    \end{equation}
    Thus, Eq.~\ref{eq:function} is equivalent to the following equation.
    \begin{equation} \label{eq:p1e5}
        \mu_{\mathbf{g}}(\mathbf{z}^*) = \mu_{\mathbf{f}}(\mathbf{z}^*)
    \end{equation}
    To prove Eq.~\ref{eq:p1e5}, it is needed to find the proper kernel functions formulas $\mathbf{K}_{\mathbf{f}},\mathbf{K}_{\mathbf{g}},\mathbf{K}_{\mathbf{f*}},\mathbf{K}_{\mathbf{g*}}$. This is a trivial task to construct kernel functions since the feasibility assumption of the kernel function is weak. Especially, a deep kernel can be constructed to satisfy the condition \citep{wilson2016deep}. 
\end{proof}

Theorem~\ref{thm:1} shows that two GPs can be connected with physics equations, which is the theoretical basis of the proposed generalized PRGP. This is substantially different from the previous study \citep{yuan2021macroscopic} because this paper does not leverage the PDE and latent force models. In the previous study, the second GP is created by applying the linear or nonlinear operator on the first GP, and the second GP is basically the latent force. 

The posterior regularization is based on optimizing the parameters to maximize the evidence lowerbound (ELBO) of the posterior across the GP and the physical knowledge GP \citep{ganchev2010posterior}.
The ELBO of the proposed PRGP includes the model posterior on data and a penalty term that encodes the physics knowledge constraints over the posterior of the variables to encourage consistency with the equations.
Jointly maximizing the penalty term in ELBO can be viewed as a soft constraint over the pure GP model, therefore, estimating the PRGP model is equivalent to estimating the pure GP model with constraints on its posterior \citep{yuan2021macroscopic}.
To provide the theoretical basis of the inference algorithm, Theorem~\ref{thm:2} shows the formulation of the approximate ELBO $\mathcal{L}$ of the PRGP model.

\begin{theorem}\label{thm:2}
    The parameter inference of the PRGP model is to maximize the approximate ELBO $\mathcal{L}$ in Eq.~\ref{eq:lllhlb_traffic} regarding the parameters defined in Eq.~\ref{eq:theta} given the input variables are the observed data $\mathcal{D}=(\mathbf{X,Y})$.    
    \begin{equation}
        \label{eq:lllhlb_traffic}
        \max \mathcal{L}= \sum_{l=1}^{d^\prime} \log\Big[\mathcal{N}([\mathbf{Y}]_l|[\mathbf{\mu_f}]_l,[\mathbf{\sigma_f}]_l)\Big]
        +\sum_{w=1}^W\gamma_w \mathbb{E}_{p(\mathbf{Z})} \mathbb{E}_{p(\mu_{\mathbf{f}_w}|\mathbf{Z},\mathbf{X},\mathbf{Y})} [\log \Big[\mathcal{N}(\Phi \mu_{\mathbf{f}_w}|\mu_{\mathbf{g}_w},\sigma_{\mathbf{g}_w})]\Big]
    \end{equation}
    where
    \begin{equation}
        \mathbf{\sigma_f}=\mathbf{K_f(X,X)}+\tau^{-1}\mathbf{I}
    \end{equation}
    \begin{equation}
        \mathbf{\sigma_g}=\mathbf{K_g(Z,Z)}
    \end{equation}
    \begin{equation}
        \theta = 
        \begin{bmatrix}
        \mathbf{\theta_f} &    \mathbf{\theta_g}
        \end{bmatrix}^\intercal
        =
        \begin{bmatrix}
        \bar{\tau} & \eta &  \tau & \nu& \delta& \kappa& v_f& \rho_{cr}& \alpha & \cdots
        \end{bmatrix}^\intercal
        \label{eq:theta}
    \end{equation}
\end{theorem}

\begin{proof}
    Generally, the ELBO of a posterior probability is yielded by analyzing a decomposition of the Kullback-Leibler (KL) divergence \citep{bishop2006pattern}. The idea of the proof is to find a tractable approximate ELBO of posterior probability $p(\mathbf{Y},\omega|\mathbf{X})$, where a positive parameter $\gamma$ is used to control the strength of regularization effect.
    The posterior probability $p(\mathbf{Y},\omega|\mathbf{X})$ is decomposed into $p(\mathbf{Y}|\mathbf{X})$ and $p(\omega|\mathbf{X},\mathbf{Y})^\gamma$.
    \begin{equation}\label{eq:post}
        p(\mathbf{Y},\omega|\mathbf{X})=p(\mathbf{Y}|\mathbf{X})p(\omega|\mathbf{X},\mathbf{Y})^\gamma
    \end{equation}
    
    First, $p(\mathbf{Y}|\mathbf{X})$ is the posterior probability of the pure GP, which is obtained with the propriety of GP.
    \begin{equation}\label{eq:pyx}
        p(\mathbf{Y}|\mathbf{X})=\mathcal{N}(\mathbf{Y}|\mathbf{\omega},\sigma_\mathbf{f})
    \end{equation}
    
    Second, by marginalizing out all the latent variables $\mathbf{g,\mu_{\mathbf{g}},Z}$ in $p(\mathbf{Y},\omega,\mathbf{g}, \mu_{\mathbf{g}},\mathbf{Z}|\mathbf{X})$ to yield $p(\omega|\mathbf{X},\mathbf{Y})$.
    \begin{equation}
        \label{eq:oyx}
        \begin{split}
            p(\omega|\mathbf{X},\mathbf{Y})&=\int_{\mathbf{Z}}\int_{\mathbf{g}}\int_{\mathbf{\mu_f(Z)}} p(\mathbf{Y},\omega,\mathbf{g}, \mu_{\mathbf{g}},\mathbf{Z}|\mathbf{X})\\
            &=\mathbb{E}_{p(\mathbf{Z})}\mathbb{E}_{p(\mu_\mathbf{f}(\mathbf{Z})|\mathbf{Z},\mathbf{X},\mathbf{Y})}\mathcal{N}(\Phi \mu_\mathbf{f}|\mathbf{\omega},\sigma_\mathbf{g})\\
        \end{split}
    \end{equation}

    Third, take the logarithm function on the both sides of Eq.~\ref{eq:post} and substitute Eq.~\ref{eq:pyx} and Eq.~\ref{eq:oyx} to yield Eq.~\ref{eq:lllh}. Note that the expectation term in Eq.~\ref{eq:lllh} brings the intractability.
    \begin{equation}
        \label{eq:lllh}
        \begin{split}
            \log [p(\mathbf{Y},\mathbf{\omega}|\mathbf{X})]=&\log [p(\mathbf{Y}|\mathbf{X})] + \gamma\log [p(\omega|\mathbf{X},\mathbf{Y})]\\
            =&\log [\mathcal{N}(\mathbf{Y}|\mathbf{\omega},\sigma_\mathbf{f})]
            +\gamma \log [\mathbb{E}_{p(\mathbf{Z})}\mathbb{E}_{p(\mu_\mathbf{f}(\mathbf{Z})|\mathbf{Z},\mathbf{X},\mathbf{Y})}\mathcal{N}(\Phi \mu_\mathbf{f}|\mathbf{\omega},\sigma_\mathbf{g})]
        \end{split}
    \end{equation}
    
    Forth, since the logarithm function is concave on its domain, the Jensen's inequality is used to find the evident lowerbound of the intractable expectation term in Eq.~\ref{eq:lllhlb}.
    \begin{equation}
        \label{eq:lllhlb}
        \begin{split}
            \log [p(\mathbf{Y},\mathbf{\omega}|\mathbf{X})]\geq \mathcal{L}= & \log [\mathcal{N}(\mathbf{Y}|\mathbf{\omega},\sigma_\mathbf{f})]
            +\gamma \mathbb{E}_{p(\mathbf{Z})}\mathbb{E}_{p(\mu_\mathbf{f}(\mathbf{Z})|\mathbf{Z},\mathbf{X},\mathbf{Y})}\log[\mathcal{N}(\Phi \mu_\mathbf{f}|\mathbf{\omega},\sigma_\mathbf{g})]\\
        \end{split}
    \end{equation}
\end{proof}
Theorem~\ref{thm:2} shows the inference of the proposed generalized PRGP is to maximize the approximate evident lowerbound of the posterior probability, which is the theoretical basis of the proposed algorithm. The formulation of the objective is similar to that in the previous study \citep{yuan2021macroscopic}, which shows the partial findings in the previous study are consistent with the proposed generalized PRGP.
Then, a critical question should be answered before encoding the physics equations in the generalized PRGP model: which kind of physics models is sufficient to be incorporated in the PRGP? To address this issue, Theorem~\ref{thm:3} shows a sufficient condition of the applicability of the physics equation in the proposed generalized PRGP. 
\begin{theorem}\label{thm:3}
    The approximate ELBO $\mathcal{L}$ is differentiable in all orders or differentiable in high orders regarding the kernel parameter $\eta$ (namely, at least in high orders) is a sufficient condition of the applicability of the physics equation in the PRGP model.
\end{theorem}
\begin{proof}
    If the physics equation is applicable in the PRGP model, the penalty term is differentiable regarding the kernel parameter $\eta$ in all orders or differentiable in high orders.
    Obviously, $p(\mathbf{Y|X})$ is differentiable regarding the kernel parameter $\eta$. The kernel function $K(\eta)$ is differentiable in all orders or high orders. This is because the kernel function is assumed to positive-definite, smooth and has derivatives of all orders in its domain.
    \begin{equation}
        \frac{\partial \mathcal{L}}{\partial \eta} = \frac{\partial \mathcal{L}}{\partial K}\frac{\partial K}{\partial\eta}
    \end{equation}
    Thus, approximate ELBO $\mathcal{L}$ is differentiable in all orders or differentiable in high orders regarding the kernel parameter $\eta$.
\end{proof}

Theorem~\ref{thm:3} shows the physics equations should be formulated so that the objective function is differentiable regarding the parameters in the proposed generalized PRGP. The physics equations are considered as a linear combination of the basic mathematical operators, such as arithmetic and and differential operators. However, the applicable operators of encoded physics equations are not specified. Thus, Corollary~\ref{cor:1} shows a few frequently used operators are applicable in the physics equations in the proposed generalized PRGP. If these operators are used in the physics equations only, the generalized PRGP is able to be inferred. And it is found the frequently used macroscopic traffic models can be formulated only with the listed operators. Note that it is only a necessary condition that the operators of physics equations are in the list, and whether the unlisted operators can be incorporated is not yet proven.
\begin{corollary}\label{cor:1}
    The necessary condition of the applicability of the physics models in the PRGP is that the physics models are composed with a subset of arithmetic, differential, comparison and disjunction operators.
\end{corollary}
\begin{proof}
    The idea of the proof is to show that the derivative $\partial \mathcal{L}/\partial \eta$ can be computed through some operators by the Chain Rule of Differentiation. The possible cases of the operator $\epsilon$ are explained one by one as follows.\\
    (a) The arithmetic operators (plus, minus, multiply, divide) are differentiable. This is proven by the sum, product and quotient rules of differentiation.\\
    (b) The differentiation operator $\epsilon$ is differentiable in all orders. The derivative is shown in the following equation .
    \begin{equation}
        \frac{\partial \mathcal{L}}{\partial \eta} = \frac{\partial \mathcal{L}}{\partial \epsilon} \frac{\partial \epsilon}{\partial \eta}
    \end{equation}
    The derivative $\partial \epsilon/\partial \eta$ is one order higher than $\epsilon$ itself. It requires that $\mathcal{L}$ is differentiable at least in high orders. In traffic models, the differentiation operator $\epsilon$ is in low orders (one or two). Thus, the differentiation requirement is satisfied\\
    (c) The physics model has a limited number of disjunction operators. In this case, each disjunction segment should be differentiable at least in high orders. And the non-differentiable points shall not cause numerical problems. \\
    (d) If the physics model has a comparison operator (greater, less, greater or equal, less or equal), the physics inequalities can be converted to equations with slack or surplus variables. The slack and surplus variables can be a part of the remainder GP $\mathbf{g}$.\\
    (e) If the term $\epsilon$ is non-differentiable, let $\partial \epsilon / \partial \eta = 1$. This setting is used to prevent any partial non-differentiable term to disable the other differentiable terms.\\
    Thus, if the physics equation is composed with a subset of the aforementioned operators (a)-(e) but not all with (e), the gradient $\partial \mathcal{L}/\partial \eta$ is related to $\eta$. 
\end{proof}

\subsection{Encoding the Discretized Traffic Flow Model}
In this section, the discretized traffic flow model is used as an example to present the generalized equation configuration and encoding technique. By re-basing the PRGP model, it is found the discretized traffic models with the two-step  procedure to avoid a substantial change to the previous PRGP method: (1) linking several neighboring inputs and the corresponding outputs via the GP; and (2) calculating the right-hand-side remainder via the discretized physics equations.  In the discretized model, Fig.~\ref{fig:encoding} shows how to reformulate the generalized physics equations (e.g. the discrete traffic flow model) into a generative component for regularization, where the nodes represents GPs; the arrows represents the stochastic conditional dependency between GPs; and the equations above the arrows show the computational transition from one GP to another.

\begin{figure}[H]
    \centering
    \includegraphics[width=0.6\textwidth]{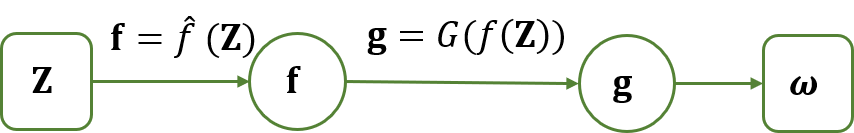}
    \caption{Encoding generalized physics equations into Gaussian process}
    \label{fig:encoding}
\end{figure}

In Fig.~\ref{fig:encoding}, the input vector $\mathbf{Z}$ with the length of $m$ has similar structure of the data input vector $\mathbf{X}$. For the convenience of computation, we further introduce a set of $m$ pseudo observations, $\mathbf{\omega}=[\omega_1,\ldots,\omega_m]^\intercal$, as dummy outputs. The pseudo observation pair $\mathbf{Z,\mathbf{\omega}}$ has the same structure with the data observation pair $\mathbf{X,Y}$, and is designed to encode the physics equations into GP. The pseudo observations $\mathbf{\omega}$ are dummy outputs of the regularization component of the stochastic model. $\mathbf{\omega}$ is used to formulate a valid Bayesian stochastic model, does not have physics meaning, and the value of $\mathbf{\omega}$ can be a vector of any constant value (i.e. the vector of $\mathbf{0}$ in this study). 

In the METANET model, the physics equations are related to four neighboring inputs, $\mathbf{Z_{0,0}}$,$\mathbf{Z_{0,1}}$,$\mathbf{Z_{-1,0}}$,and $\mathbf{Z_{1,0}}$, in time and space, and the corresponding outputs, $ \mathbf{\hat{f}(Z_{0,0}),\hat{f}(Z_{0,1}),\hat{f}(Z_{-1,0}),\hat{f}(Z_{1,0})}$, are estimated for yielding the resultant right-hand-side value $\mathbf{g}$ in Eq.~\ref{eq:gz1}, where the subscript refers to the difference in elements of the input vector $\mathbf{z}=[i, k]$. For example, if the element in the input matrix $\mathbf{Z_{0,0}}$ is $[i, k]$, the corresponding element in $\mathbf{Z_{0,1}}$ is $[i, k+1]$. Eq.~\ref{eq:gz2} shows the equivalent formation of Eq.~\ref{eq:gz1}, where each row of the equation corresponds to Eqs.~\ref{eq:wp1gp}-\ref{eq:wp5gp}, respectively.
\begin{equation}
    \mathbf{G[\hat{f}(Z_{0,0}),\hat{f}(Z_{0,1}),\hat{f}(Z_{-1,0}),\hat{f}(Z_{1,0})]=g}
    \label{eq:gz1}
\end{equation}
\begin{equation}
    \begin{bmatrix}
        G_1\Big[\mathbf{\hat{f}(Z_{0,0})} & \mathbf{\hat{f}(Z_{0,1})}& \mathbf{\hat{f}(Z_{-1,0})}& &\Big] \\
        G_2\Big[\mathbf{\hat{f}(Z_{0,0})} & \mathbf{\hat{f}(Z_{0,1})}& \mathbf{\hat{f}(Z_{-1,0})}& \mathbf{\hat{f}(Z_{1,0})}&\Big]\\
        G_3\Big[\mathbf{\hat{f}(Z_{0,0})} & & & &\Big]
    \end{bmatrix}
    =
    \begin{bmatrix}
    \mathbf{g}_1\\
    \mathbf{g}_2\\
    \mathbf{g}_3
    \end{bmatrix}
    \label{eq:gz2}
\end{equation}

The traffic flow model METANET is reformulated to the functions of estimations in Eqs.~\ref{eq:wp1gp}-\ref{eq:wp5gp}. The encoded physics equations do not have to be the exactly same formulations. The following modifications are made to accommodate the traffic flow model in the PRGP framework. (a) The random error terms $\xi^v_{i,k}, \xi^q_{i,k}$ are removed since the GP already captures the random errors. (b) The on-ramp off-ramp flows, $r_{i,k}, s_{i,k}$, are assumed to be not observed, and are removed in Eq.~\ref{eq:wp1gp}, and those unobserved measures and random noise are captured by the right-hand side term $g_1$. (c) For the implementation concern, a small number is also added to the denominators in Eqs.~\ref{eq:wp1gp}-\ref{eq:wp5gp} to prevent the numerical problem.

\begin{equation}
    G_1\Big[\hat{f}(\mathbf{z}_{0,0}), \hat{f}(\mathbf{z}_{0,1}),\hat{f}(\mathbf{z}_{-1,0})\Big] =
    \hat{\rho}_{i,k+1}-\hat{\rho}_{i,k}-\frac{T}{\Delta_i\lambda_i}[\hat{q}_{i-1,k}-\hat{q}_{i,k}]=g_1
    \label{eq:wp1gp}
\end{equation}

\begin{equation}
\begin{split}
    G_2\Big[\hat{f}(\mathbf{z}_{0,0}), \hat{f}(\mathbf{z}_{0,1}),\hat{f}(\mathbf{z}_{-1,0}), \hat{f}(\mathbf{z}_{1,0})\Big] =&
    \hat{v}_{i,k+1}-\hat{v}_{i,k}-\frac{T}{\tau}[V(\hat{\rho}_{i,k})-\hat{v}_{i,k}]\\
    &-\frac{T}{\Delta_i}\hat{v}_{i,k}(\hat{v}_{i-1,k}-\hat{v}_{i,k})
    +\frac{\sigma T}{\tau\Delta_i}\frac{\hat{\rho}_{i+1,k}-\hat{\rho}_{i,k}}{\hat{\rho}_{i,k}+\kappa}
    = g_2
\end{split}
    \label{eq:wp3gp}
\end{equation}

\begin{equation}
    G_3\Big[\hat{f}(\mathbf{z}_{0,0})\Big] = 
    \hat{q}_{i,k}-\hat{\rho}_{i,k}\hat{v}_{i,k}\lambda_{i}
    = g_3
    \label{eq:wp5gp}
\end{equation}

\subsection{Implementation}
Before estimating the traffic state, the parameters of the generalized PRGP model should be learned with given observed data. In the original ELBO formulation, the strength of the regularization is related to the parameter $\gamma$ and the numerical value of the regularization term. The numerical problem can be caused by the improper value of $\gamma$ and the random error of the regularization term. To address this problem, the inference problem is discomposed into two alternating stochastic optimization problems, as shown in Theorem~\ref{thm:4}. 

\begin{theorem}\label{thm:4}
    The parameter inference problem of the PRGP model is equivalence to two alternating stochastic optimization problems. In the first problem, the input variables are the observed data $\mathcal{D}=\mathbf{(X,Y)}$, and the objective function is to maximize the $\mathcal{L}_f$. In the second problem, the input variables are the random pseudo-inputs $\mathbf{(Z,\hat{f})}$, and the objective function is to maximize $\mathcal{L}_g$, where $\mathcal{L}_{f}$ and $\mathcal{L}_{g}$ are denoted as the partial terms of $\mathcal{L}$, as shown in the following equations.
    \begin{equation}
        \mathcal{L}_f=
        \sum_{l=1}^{d^\prime} \log\Big[\mathcal{N}([\mathbf{Y}]_l|[\mathbf{\mu_f}]_l,[\mathbf{\sigma_f}]_l)\Big]
    \end{equation}    
    \begin{equation}
        \mathcal{L}_g=
        \sum_{w=1}^W\gamma_w \mathbb{E}_{p(\mathbf{Z})} \mathbb{E}_{p(\mu_{\mathbf{f}_w}|\mathbf{Z},\mathbf{X},\mathbf{Y})} \log \Big[\mathcal{N}(\Phi \mu_{\mathbf{f}_w}|\mu_{\mathbf{g}_w},\sigma_{\mathbf{g}_w})\Big]
    \end{equation}
\end{theorem}
\begin{proof}
    In the original inference problem, the trainable parameters, $\theta$, can be updated by the following equation.
    \begin{equation}
        \label{eq:gradient}
        \theta^{t+1}=\theta^{t}+\phi\nabla_\theta\mathcal{L}
    \end{equation}
    It is trivial to find the step-size $\phi_f, \phi_g$ so that the following equation holds.
    \begin{equation}
        \label{eq:gradient2}
        \phi\nabla_\theta\mathcal{L} = \phi_f\nabla_\theta\mathcal{L}_{f}+\phi_g\nabla_\theta\mathcal{L}_{g}
    \end{equation}    

Fig.~\ref{fig:alg} depicts one iteration in the high dimensional parameter space to illustrate the design concept of the proposed posterior regularization algorithm for the proposed model. In Fig.~\ref{fig:alg}, the parameter space consists of the two dimensions of the outputs (i.e. flow and speed $q,v$); the dots show the vector of parameters $\theta^{(t)}$ is updated to the new vector $\theta^{(t+1)}$ via the auto-differentiation in the $t^{th}$ iteration; the arrows show the directions of the gradients of the objective function (i.e. evidence lowerbound of posterior probability); the blue arrows represent the increments via the conventional GP in two dimensions, the green arrow shows the increment via the proposed physical knowledge regularizer, and the red arrow is the resultant sum of increments. 

\begin{figure}[H]
    \centering
    \includegraphics[width=0.4\textwidth]{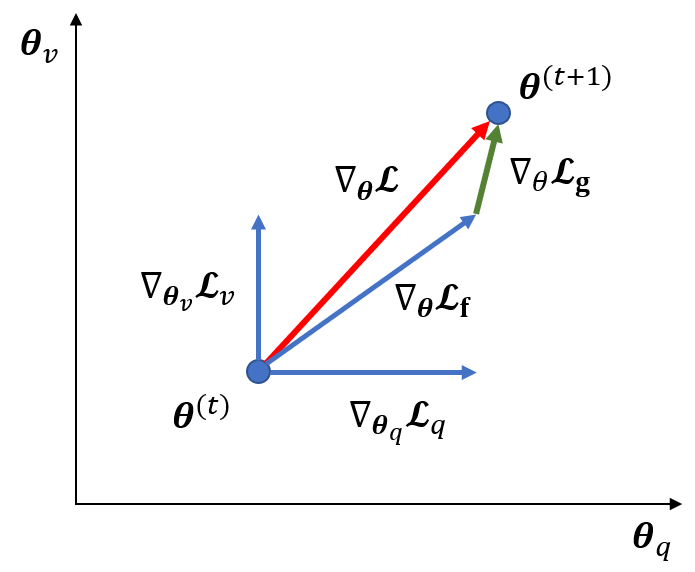}
    \caption{The posterior regularization algorithm for the proposed model}
    \label{fig:alg}
\end{figure}
\end{proof}
Theorem~\ref{thm:4} shows the iteration on the objective function of the proposed generalized PRGP is equivalent to iterate on its two linear components, which is the theoretical basis of the proposed solution algorithm.
Then, the procedure of implementing the alternating stochastic optimization is shown as follows. The stopping criteria include (a) the number of iterations exceeds a prefixed value, and (b) the difference of the objective value $\mathcal{L}_{f}^{(t+1)} - \mathcal{L}_{f}^{(t)}$ is 0 for more than a prefixed number of iterations. \\  

\begin{algorithmic}[1]
    \label{alg:1}
    \STATE Initialize the computational graph and parameters
    \WHILE{not reach stopping criteria}
        \STATE Sample a set of input locations $\mathbf{Z}$
        \STATE Estimate the posterior target function values  $\mathbf{\hat{f}(Z)}$
        \STATE Compute $G_1,G_2,G_3$ in Eq.~\ref{eq:wp1gp}-\ref{eq:wp5gp}
        \STATE Compute ELBO ${\mathcal{L}}=[{\mathcal{L}}_f, {\mathcal{L}}_g]^\intercal$ with samples $\mathbf{(X, Y),(Z, \hat{f}(Z))}$
        \STATE Compute the gradients $\nabla_\theta{\mathcal{L}}_{f},\nabla_\theta{\mathcal{L}}_{g}$
        \STATE Update the parameters  $\theta^{(t+1)}=\theta^{(t)}+\phi_f\nabla_\theta{\mathcal{L}}_{f}+\phi_g\nabla_\theta{\mathcal{L}}_{g}$
    \ENDWHILE
    \STATE Output learned parameters $\theta$
\end{algorithmic}

\vspace{0.1in}

To solve this problem, the inference algorithm is implemented in the open-source auto-differentiable computational graph framework, \emph{Tensorflow}, where the optimizer ADAM \citep{kingma2014adam} is chosen for updating the parameters by rule-of-thumb. Note that the implementation does not rely on the specific framework, and the comparable libraries are potentially feasible as well. 
Before computing the gradients on $\mathcal{L}$, the auto-differentiation tool first creates a computational graph for all data, parameters, and operators. Fig.~\ref{fig:compute} depicts the computational graph, where the vertices represent for the variables (i.e. scalars, matrices, or tensors), the circle vertices involve trainable parameters, the squared vertices represents the estimation, the rounded rectangles are for the data set; the arrows represent the equation calculation; the blue vertices and arrows are for the original GP, and the green vertices are for the physics regularizer. $\hat{\mathbf{K}}_w$ denotes the kernel function of the $w^{th}$ equation. For the convenience of representation, $\mathcal{L}_{v}$, $\mathcal{L}_{q}$, and $\mathcal{L}_w$ are denoted as the part of the objective function $\mathcal{L}$, and they are defined as follows, where $v$ represents the velocity, $q$ represents the traffic flow, $w$ represents the index of the equation. 
\begin{equation}\label{eq:lv}
\mathcal{L}_{v}= \log\Big[\mathcal{N}([\mathbf{Y}]_1|[\mathbf{\mu_f}]_1,[\mathbf{\sigma_f}]_1)\Big]
\end{equation}
\begin{equation}\label{eq:lq}
    \mathcal{L}_{q}=\log\Big[\mathcal{N}([\mathbf{Y}]_2|[\mathbf{\mu_f}]_2,[\mathbf{\sigma_f}]_2)\Big]
\end{equation}
\begin{equation}\label{eq:lw}
    \mathcal{L}_{w}=\mathbb{E}_{p(\mathbf{Z})} \mathbb{E}_{p(\mu_{\mathbf{f}_w}|\mathbf{Z},\mathbf{X},\mathbf{Y})} \log \Big[\mathcal{N}(\Phi \mu_{\mathbf{f}_w}|\mu_{\mathbf{g}_w},\sigma_{\mathbf{g}_w})\Big]
\end{equation}

The computational graph shows the computational dependency of the variables so that each vertex is computed from a function of precursive variables. Given the computational graph, the auto-differentiation libraries can find the gradient of the objective function for optimizing the trainable parameters iteratively.

\begin{figure}[H]
    \centering
    \includegraphics[width=0.6\textwidth]{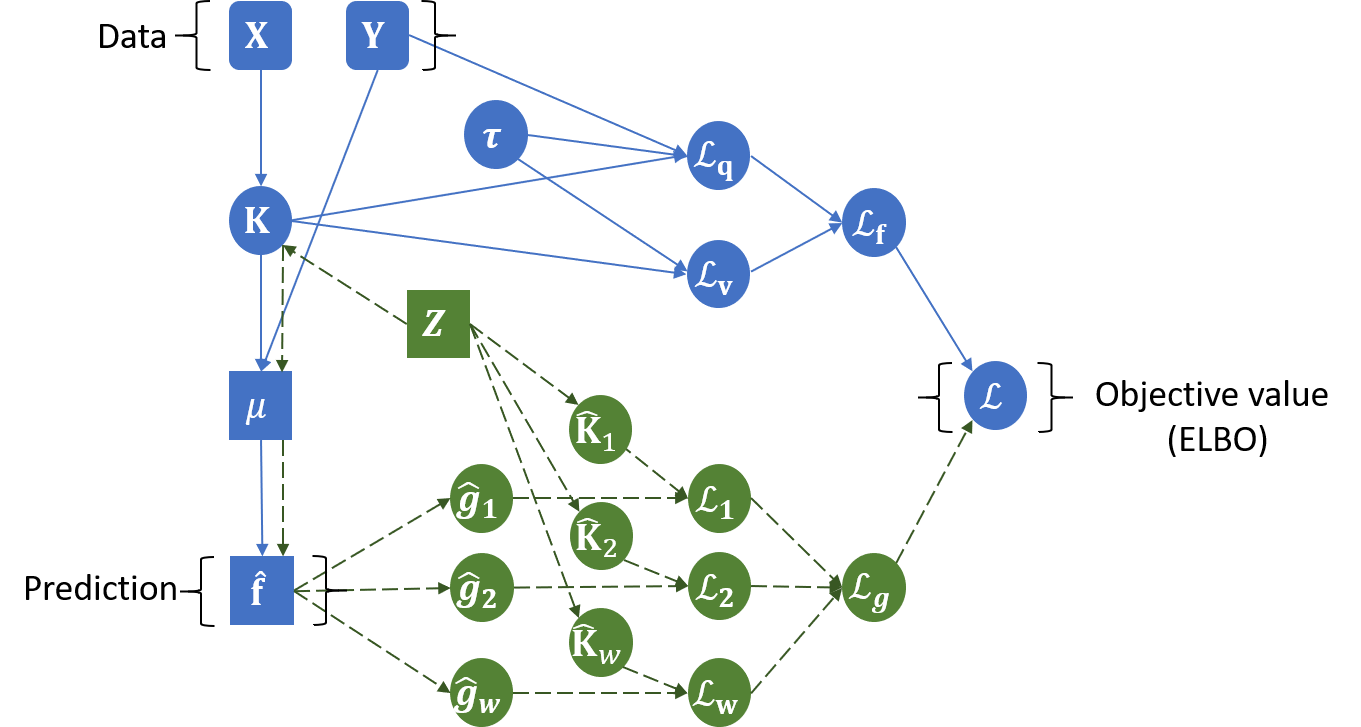}
    \caption{The computational graph of the estimation and the objective function}
    \label{fig:compute}
\end{figure} 

The computational complexity is cubic of the product of the sample size and the output dimension $O((nd^\prime)^3+m^3)$. By applying approximate GP, the computational complexity can be reduced to $O((nd^\prime)^2*\zeta+m^3)$, where $\zeta$ is a constant \citep{liu2020gaussian}.

\section{Numerical Examples and Model Evaluations}\label{sec:4}
\subsection{Data Collection}
To evaluate the performance of the proposed method, we applied the PRGP method to estimate the traffic flow in a stretch of the interstate freeway I-15 across Utah, U.S.
The Utah Department of Transportation (UDOT) has installed sensors every a few miles along the freeway.
Each sensor counts the number of vehicles passed every minute, measures the speed of each vehicle, and sends the data back to a central database, called Performance Measurement System (PeMS).
The collected real-time data and road conditions are available online and can be accessed by the public.
Various data spans in spatial and temporal dimensions are tested. In the studied scenario, the separate freeway segment in I-15 has $4$ detectors.
The data was collected from August 5, 2019 to August 19, 2019. Since the data is collected every $5~min$, there are $288$ observations per detector per day.
The studied stretch is illustrated in Fig.~\ref{fig:stretch1}, where the blue bars represent the locations of traffic detectors. 
To better illustrate the fluctuation of traffic in space and time, Fig.~\ref{fig:groundtruth} plots the distribution of speeds and flows. The speed drops are caused by the sudden congestion near the ramps. By comparing the speed pattern among different days, we can observe similar drops frequently during the peak-hours.

\begin{figure}[H]
    \centering
    \includegraphics[width=0.8\textwidth]{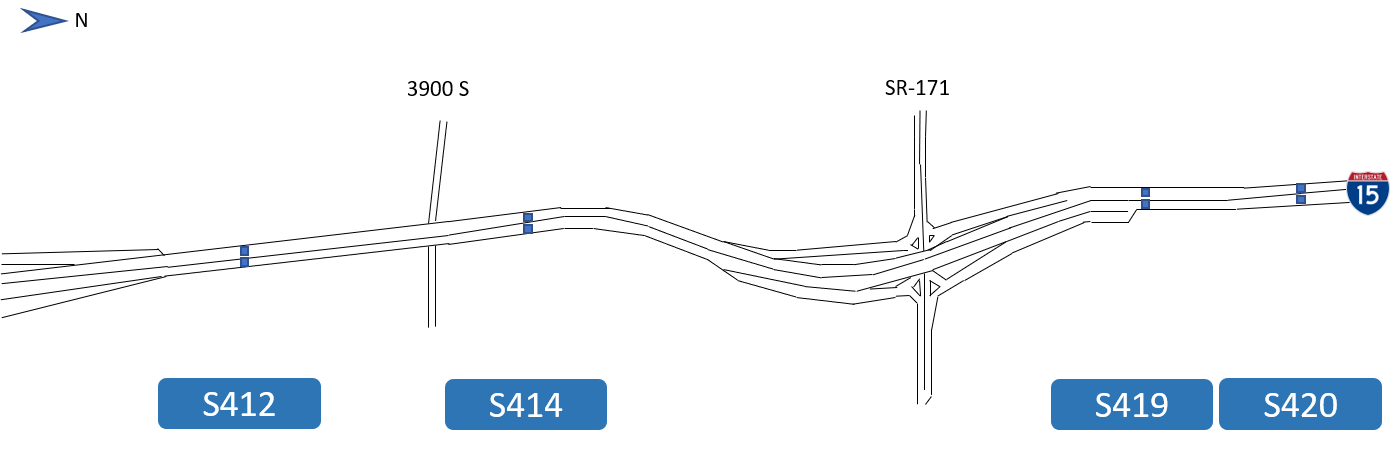}
    \caption{The stretch of the studied freeway segment which includes 4 detectors}
    \label{fig:stretch1}
\end{figure}
\begin{figure}[H]
    \centering
        \begin{subfigure}[b]{0.6\textwidth}
        \centering
        \includegraphics[width=\textwidth]{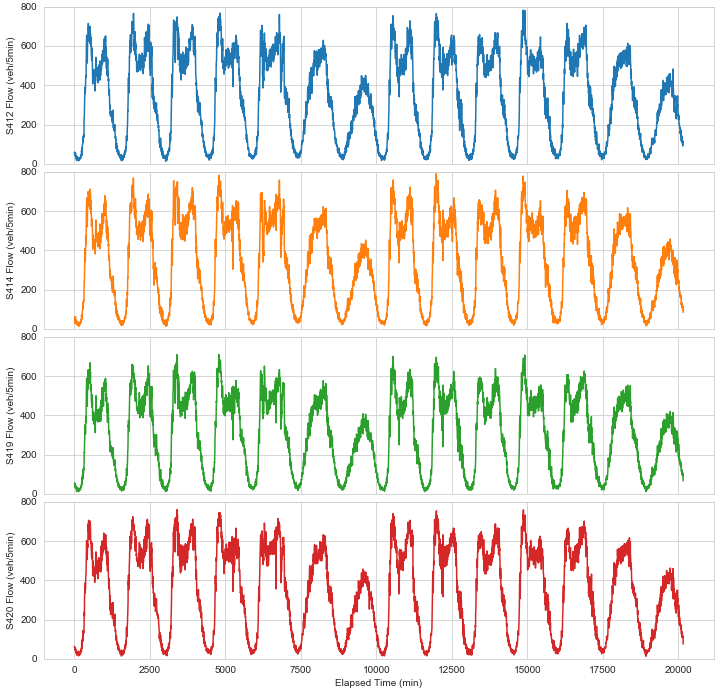}
        \caption{Flow}
    \end{subfigure}
    \end{figure}
 \begin{figure}[H] \ContinuedFloat
  \centering
    \begin{subfigure}[b]{0.6\textwidth}
        \includegraphics[width=\textwidth]{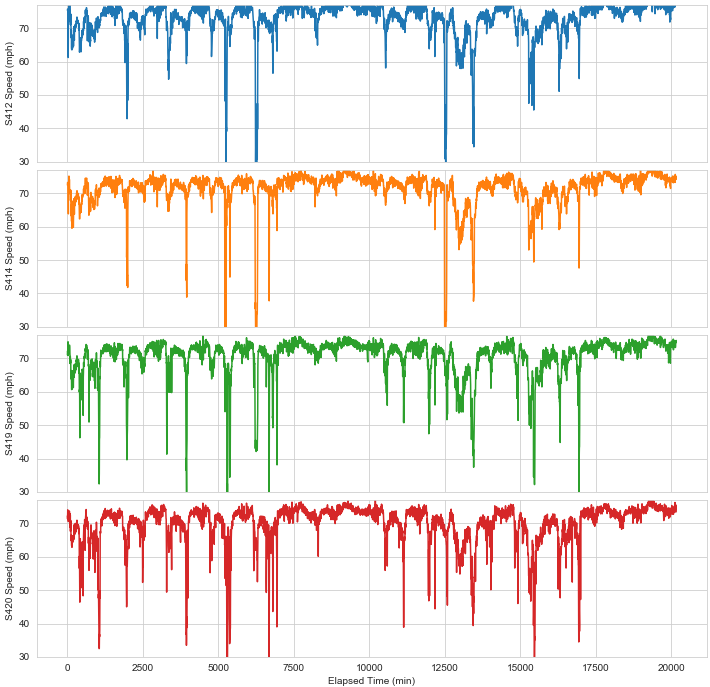}
        \caption{Speed}
    \end{subfigure}
    \caption{The ground truth of the flow and speed in the studied case}
    \label{fig:groundtruth}
\end{figure}

\subsection{Baseline methods}
In this paper, "METANET" represents the off-line calibrated fixed parameter METANET model, "METANET-EKF" refers to the extended Kalman filter for online correcting the estimated flow and speed of METANET model. Herein, Kalman filter and its extensions deal with a series of measurements observed over time considering the random measure error \citep{mihaylova2006unscented,work2008ensemble,wang2022real}.

To prove the superiority of the proposed PRGP compared with the pure ML method and the physics models, this section aims to compare the proposed PRGP method with calibrated deterministic baseline model (METANET) \citep{papageorgiou1989macroscopic}, the Extended Kalman filter (EKF) on the stochastic model (METANET-EKF) \citep{wang2022real}, the Gaussian Process model, and several other pure ML models. The parameters of the key notations of the METANET and EKF have been calibrated with the field data. It should be noted that we selected the METANET and METANET-EKF as the baseline models due to (1) they are based on the same modeling foundation - discretized 2nd order traffic flow model; and (2) they are commonly used in the traffic flow studies since they are most representative and easy to follow. 

The parameters of the calibrated models are listed as follows. Table~\ref{tab:param_metanet} shows the initial METANET model parameters and the parameters for EKF are listed in Table~\ref{tab:param_ekf}.
The calibrated parameters of METANET only methods can be used as the parameters in PRGP. In comparison to the METANET only methods, the parameters of METANET in PRGP is more tolerate in the value. Even if the parameters of the METANET are not so well-calibrated, the PRGP can still use the encoded equations to regularize the GP and update the parameters. However, the updated parameters are not capable to be used in the METANET only method.

\begin{table}[H]
    \centering
    \caption{The initial parameters of the physical model}
    \begin{tabular}{cc}
    \toprule
    Parameter & Value (unit)\\
    \midrule
    $I$         & $20$\\
    $T$         & $1/360~(h)$ \\
    $v_f$       & $120~(km/h)$\\
    $\nu$       & $35~(km^2/h)$\\
    $\delta$    & $1.4$\\
    $\tau$      & $0.05~(h)$\\
    $\alpha$    & $1.4324$\\
    $\Delta_i$  & $0.5~(km)$\\
    $\rho_{cr}$ & $36.85~(veh/km)$\\
    $\kappa$    & $13~(veh/km)$\\
    $\lambda_i$ & $4$\\
    \bottomrule
    \end{tabular}
    \label{tab:param_metanet}
\end{table}

\begin{table}[H]
    \centering
    \caption{The initial parameters of Extended Kalman filter}
    \begin{tabular}{cl}
    \toprule
    Parameter & Value (unit)\\
    \midrule
        $D(\xi^q_{i,k})$ & $100~veh/h$ \\
        $D(\xi^v_{i,k})$ & $11~km/h$\\
        $D(\xi^q_{0,k})$ & $100~veh/h$ \\
        $D(\xi^v_{0,k})$ & $5~km/h$ \\
        $D(\xi^\rho_{11,k})$ & $1.5~veh/km/lane$ \\
        $D(\xi^r_{\Gamma,k})$ & $3~veh/h$ \\
        $D(\xi^\beta_{9,k})$ & $0.001$ \\
        $D(\gamma^q_{i,k})$ & $100~veh/h$\\
        $D(\gamma^v_{i,k})$ & $10~km/h$\\
        $D(\gamma^r_{\Gamma,k})$ & $20~veh/h$ \\
        $D(\gamma^s_{9,k})$ & $10~veh/h$ \\
        $D(\xi^{v_f}_{k})$ & $0.5~veh/h$ \\
        $D(\xi^{\rho_{cr}}_{k})$ & $0.1~veh/km/lane$ \\
        $D(\xi^a_{k})$&$0.01$\\
    \bottomrule
    \end{tabular}
    \label{tab:param_ekf}
\end{table}

Also, "Pure GP" means the Gaussian process based pure machine learning method, and "PRGP" refers to the proposed physics regularized Gaussian process with the aid of METANET. Gaussian Process (GP) is a group of multivariate normally distributed random variables indexed by time and/or space. GP has weak assumptions \citep{rasmussen2003gaussian} and is a widely-used non-parametric stochastic model in various fields, and the previous studies \citep{rodrigues2018heteroscedastic,rodrigues2018multi,neumann2009stacked,xie2010gaussian,ide2009travel,armand2013modelling,liu2013adaptive} have shown the application and effective in the traffic flow problems. 
Notably, the METANET with filtering methods and the proposed PRGP method are technically different: the filtering methods are used to correct the METANET model estimation, which is recognized as model-based methods. The PRGP method is used to regularize the GP training process, which is considered as a data-driven method. 
Other popular ML models, such as Deep Neural Network \citep{xu2020ge}, support vector machine (SVM) \citep{asif2013spatiotemporal}, random forest (RF) \citep{zhang2020freeway}, the Extreme Gradient Boosting (XGB) \citep{zhang2015gradient}, and the Gradient Boosting Decision Tree (GBDT) \citep{ma2017prioritizing}, are also tested as baselines for comparisons.
In the literature, ML is frequently referred to as a black box since its functions work in a way that inputs go in, outputs come out, but the processes between them are opaque.
This research provides the first key step to convert black-box ML methods into the grey-box models, and is elaborated in the result analysis as follows: 
(a) The difference in the PRGP involving various traffic flow models shows the impact of the physics models on the TSE results. This property of PRGP can be used to refine the estimation by using more advanced variations of traffic flow models.
(b) In comparison to the other ML models, GP is to use a linear combination of the observed data $(\mathbf{X,Y})$ to estimate the target points $\hat{f}(X^*)$ at new location and time $\mathbf{X}^*$, and the inference method of GP is derived with a tractable procedure rigorously. Thus, the GP-based methods have notable better performance among the other ML methods, and are chosen as the base methods for the PRGP extensions.
(c) In the previous study, the physics regularizer was derived from encoding of physics knowledge-related equations into GPs, which is a theoretical plausible procedure. The results of PRGP can be interpreted by comparing the encoded physics equations: the better property of the encoded physics equations, the more potentials of the PRGP estimator performance. Besides the tested METANET model, numerous unexpolored traffic flow models can be further investigated to yield improved estimation performance.

\subsection{Case Setup}
To evaluate the performance of the proposed method, the testing cases are constructed regarding the basic TSE problem with unobserved locations. Besides, to show the capability of PRGP, the testing cases are also created for the robustness with random bias and the scarceness with random missing data:
(a) To further test the robustness of methods in each case, we investigate the biased data scenarios by artificially adding high measure biases to the traffic flow in the training data to mimic the common device malfunction situations. The robustness analysis is conducted to show the capability of dealing with the unpredictable misleading inputs in the training phase. Theoretically, the proposed PRGP is more robust than pure GP on noisy dataset. To justify this feature, a certain portion of the training dataset is replaced by synthesized noise. The testing set is not polluted original data. However, it should be noted that the comparable methods, offline METANET method and EKF, for METANET are not designed to contend the biased data. In the robustness study, $50\%$ of the training data is replaced by the flawed data, which are generated with $100-veh/5min$ noise in flows and $5-mph$ noise in speeds, and the testing data keep unchanged. 
(b) In the real-world scenarios, researchers and engineers may suffer from the limited data (e.g., some data are lost). Hence, to further investigate the performance of the proposed model and the baselines under various training data size, we conduct the sensitivity analysis on various sample ratios. The tested sample ratios are $0.714,~0.357,~0.178$ corresponding to $5,760,~2,880,~1,440$ samples, respectively. 

The input variables include the location mileage of each sensor and the time of each read. In the literature, the data index representation $\mathbf{(X,Y)}$ has three major variations: (road segment, time interval), (road segment, day, time interval) and (road segment, week, day-of-week, time interval). In the experiments, we use the compatible representation (road segment, time interval), namely $(i,k)$, for consistence purpose. The traffic measures, flow $q$ and speed $v$, are employed in the training and testing because the density is directly related to these two measures and is not recorded in the original data source.
Note that the other variations of structural representation of the data are fully compatible with the proposed model, and the impact of the data representation may depend on the specific case. 

In the setup of the experiments, the prefixed parameters of the proposed method are summarized in Table~\ref{tab:prefixed}. Note that the strength of regularization $\lambda$ does not need to be fine-tuned because the gradients of the parts of the objective function can be yield separately regarding the parameters. The parameter $m$ has impact on the result, and can be fine-tuned. However, if the value $m$ is not too small (e.g. 1 or 2) to enable the pseudo-sample, the impact on the performance is limited. Considering the time complexity of the algorithm is sensitive to the value of $m$, we selected a constant small value of $m$ in each case for the testing purpose.

\begin{table}[H]
    \centering
    \caption{The prefixed parameters of the proposed method}
    \begin{tabular}{lc}
    \toprule
    Parameter & Value\\
    \midrule
        Testing set size &$576$\\
        The number of pseudo observations $m$ & $10$\\
        The number of iterations & $500$ \\
        The learning rate $\phi$ & $0.01$\\
        The number of physics equations & $3$\\
    \bottomrule
    \end{tabular}
    \label{tab:prefixed}
\end{table}

To quantify the accuracy of estimates, Rooted Mean Squared Error (RMSE) and Mean Absolute Percentage Error (MAPE) of each dimension are used as the performance metric, which are defined in Eqs.~\ref{eq:def_rmse}-\ref{eq:def_mape}. 
\begin{equation}
    RMSE_j = \sqrt{\frac{1}{n}\sum_{l=1}^{n}{\Big([\mathbf{y}_j]_l-[\hat{\mathbf{f}}_j]_l\Big)^2}}, \forall j\in {1,\ldots,d^\prime}
\label{eq:def_rmse}
\end{equation}
\begin{equation}
MAPE_j = \frac{100\%}{n}\sum_{l=1}^{n}{\Big\vert\frac{[\mathbf{y}_j]_l-[\hat{\mathbf{f}}_j]_l}{[\mathbf{y}_j]_l}\Big\vert}, \forall j\in {1,\ldots,d^\prime}
\label{eq:def_mape}
\end{equation}

\subsection{Results Analysis}
Table~\ref{tab:benchmark1} shows the results of the proposed method and the physics models.
In comparison to the physics models, most ML models except SVM and MP can obviously outperform the physics models (i.e., METANET and METANET-EKF) in terms of providing more accurate estimations of both flows and speeds. For example, the GP can yield a $63.29~veh/5min$ of RMSE and a $28.16\%$ of MAPE for flow and a $1.78~mph$ of RMSE and a $1.98\%$ for MAPE for speed, while the physics model based methods produced much higher RMSEs and MAPEs of both flow and speed estimates. Further comparison between those ML models and the PRGP models reveal that PRGP models can improve the accuracy of both flow and speed estimations. However, the improvement is not significant compared with several ML models (e.g., RF and XBDT) , which is because those ML models can already achieve a very good estimation performance and leaves limited space for improvement by the PRGP.
Also, it should be noted that the inputs of the proposed PRGP methods and classical traffic flow model are different. The latter often requires the on-ramp and off-ramp flow observations as inputs, while the proposed method assumes unobserved on-ramp and off-ramp flows in the model and does not require such data.

Fig.~\ref{fig:comp_flow} and Fig.~\ref{fig:comp_speed} compare the flow and speed estimation and the ground truth for the Case I, respectively. In each figure, the blue dot shows the estimated value versus the ground true value, and if the slope of the red trend line is closed to $1$ and the intercept is closed to $0$, the estimation result is considered to be accurate. Fig~\ref{fig:comp_flow} shows the METANET-EKF method outperforms the METANET in flow estimation, however, both of them have lower accuracy in speed estimation than GP and PRGP. The proposed PRGP has similar flow accuracy as GP, and has slightly better speed accuracy than GP.

\begin{table}[H]
    \caption{Comparison of the model results under Case I}
    \centering
    
    \begin{tabular}{p{3cm}|p{2cm}|p{2cm}|p{2cm}|p{2cm}}
    \toprule
         Method & RMSE of flow (veh/5min) & MAPE of flow & RMSE of speed (mph) & MAPE of speed \\ 
    \midrule
        METANET        & $96.17$  & $37.48\%$   & $9.11$ & $11.4\%$ \\
        METANET-EKF    & $82.48$  & $35.95\%$   & $5.74$ & $7.17\%$ \\
        SVM            & $102.15$ & $43.88\%$   & $5.58$ & $6.32\%$ \\
        RF             & $52.91$  & $15.48\%$   & $3.31$ & $3.30\%$ \\
        DNN  &  $67.57$ &  $31.24\%$   &  $4.12$ &  $2.68\%$ \\
        XGB	           & $51.24$  & $12.53\%$   & $2.73$ & $3.15\%$ \\
        GBDT	       & $58.70$  & $18.87\%$   & $3.29$ & $3.26\%$ \\
        pure GP        & $63.29$  & $28.16\%$   & $1.78$ & $1.98\%$ \\
        PRGP           & $41.32$  & $12.10\%$   & $1.55$ & $1.61\%$ \\
    \bottomrule
    \end{tabular}
    \label{tab:benchmark1}
\end{table}

\begin{figure}[H]
    \centering
        \begin{subfigure}[b]{0.45\textwidth}
        \centering
        \includegraphics[width=\textwidth]{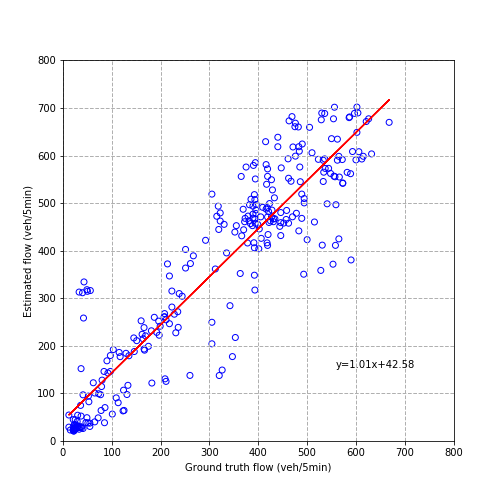}
        \caption{METANET}
        \label{fig:metanet_flow}
    \end{subfigure}
    \hfill
    \begin{subfigure}[b]{0.45\textwidth}
        \includegraphics[width=\textwidth]{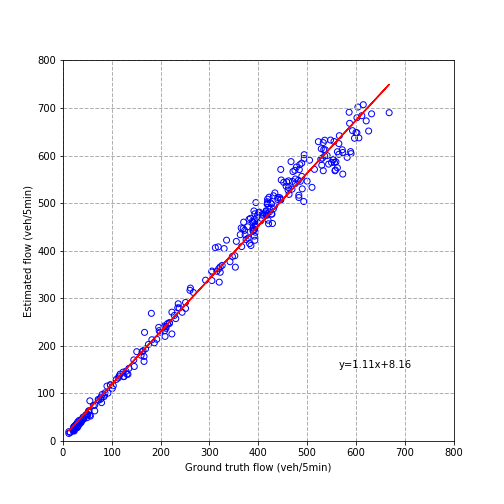}
        \caption{METANET-EKF}
        \label{fig:ekf_flow}
    \end{subfigure}
\end{figure}
\begin{figure}[H]\ContinuedFloat
    \centering
    \begin{subfigure}[b]{0.45\textwidth}
        \includegraphics[width=\textwidth]{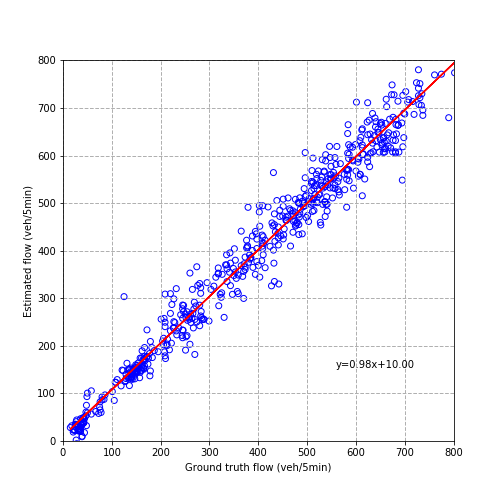}
        \caption{GP}
        \label{fig:gp_flow}
    \end{subfigure}
    \hfill
    \begin{subfigure}[b]{0.45\textwidth}
        \includegraphics[width=\textwidth]{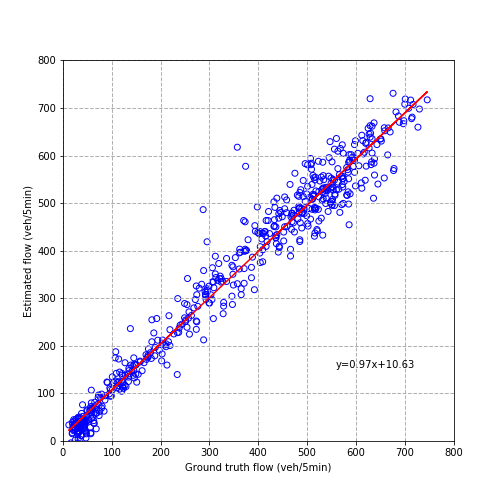}
        \caption{PRGP}
        \label{fig:prgp_flow}
    \end{subfigure}
    \caption{Comparison between flow estimations and the ground truth under Case I}
       \vspace{-0.2in}
    \label{fig:comp_flow}
\end{figure}

\begin{figure}[H]
    \centering
        \begin{subfigure}[b]{0.45\textwidth}
        \centering
        \includegraphics[width=\textwidth]{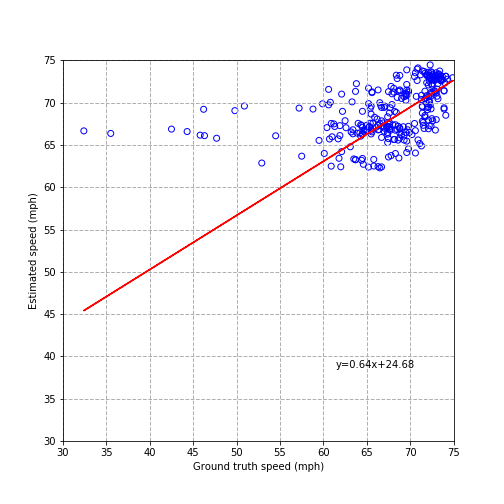}
        \caption{METANET}
        \label{fig:metanet_speed}
    \end{subfigure}
    \hfill
    \begin{subfigure}[b]{0.45\textwidth}
        \includegraphics[width=\textwidth]{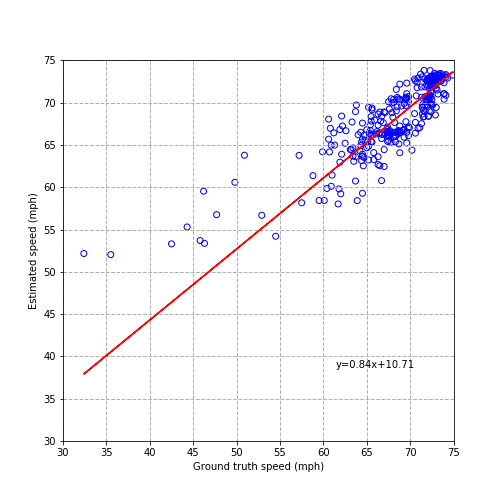}
        \caption{METANET-EKF}
        \label{fig:ekf_speed}
    \end{subfigure}
\end{figure}
\begin{figure}[H]\ContinuedFloat
    \centering
    \begin{subfigure}[b]{0.45\textwidth}
        \includegraphics[width=\textwidth]{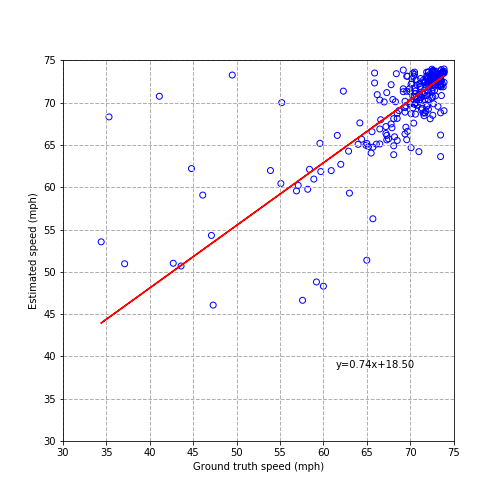}
        \caption{GP}
        \label{fig:gp_speed}
    \end{subfigure}
    \hfill
    \begin{subfigure}[b]{0.45\textwidth}
        \includegraphics[width=\textwidth]{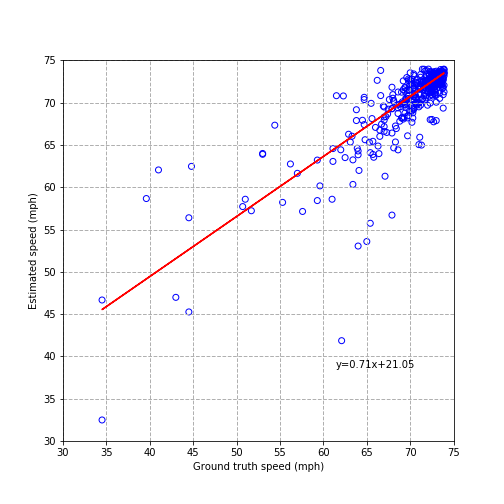}
        \caption{PRGP}
        \label{fig:prgp_speed}
    \end{subfigure}
    \caption{Comparison between speed estimation and the ground truth under Case I}
       \vspace{-0.2in}
    \label{fig:comp_speed}
\end{figure}

Notably, flow estimation at the locations without observations is a challenging task. For example, for the baseline method, METANET-EKF, the relative error of speed was ranged from $14\%$ to $16\%$, and the relative error of density was ranged from $21\%$ to $43\%$. They did not report the error of estimated traffic flow, but we can roughly estimate MAPE of traffic flow may range from $35\%$ to $60\%$. Hence, the results of METANET and METANET-EKF Table~\ref{tab:benchmark1} should be reasonable and proposed model can greatly improve the estimation accuracy. 

Table~\ref{tab:benchmark5} shows PRGP can achieve better estimation when the training data set changes from small to large (with different sample ratio), compared with those ML-based baselines. Notably, the model-based methods, METANET and METANET-EKF can not be adopted in this case due to the incomplete input patterns. Also, it can be observed that with the reduction of the sample ratio, the proposed PRGP model can still yield acceptable estimation results (e.g., $45.87~veh/5-min$ RMSE of flow) while the performance of the other models have been downgraded significantly.

\begin{table}[H]

    \caption{Comparison of model results under various sample ratios in Case I}
    \centering
    
    \begin{tabular}{p{3cm}|p{2cm}|p{2cm}|p{2cm}|p{2cm}|p{2cm}}
    \toprule
        Method & Sample ratio & RMSE of flow (veh/5min) & MAPE of flow & RMSE of speed (mph) & MAPE of speed \\
    \midrule
        SVM         & 0.714 & $91.00$  & $37.71\%$  & $4.38$ & $4.11\%$\\
        RF          & 0.714 & $42.80$  & $11.39\%$  & $3.05$ & $2.96\%$\\
        DNN          & 0.714 & $43.21$  & $11.54\%$  & $2.31$ & $1.98\%$\\
        XGB	        & 0.714 & $42.08$  & $11.24\%$  & $3.59$ & $4.03\%$\\
        GBDT	    & 0.714 & $44.89$  & $11.64\%$  & $3.20$ & $3.05\%$\\
        pure GP     & 0.714 & $63.31$  & $28.16\%$  & $1.63$ & $1.55\%$\\
        PRGP        & 0.714 & $42.02$  & $11.40\%$  & $1.52$ & $1.45\%$\\
    \midrule
        SVM         & 0.357 & $96.33$  & $34.09\%$  &$4.49$  & $4.28\%$\\
        RF          & 0.357 & $55.82$  & $15.06\%$  &$3.69$  & $3.77\%$\\
        DNN         & 0.714 & $53.12$  & $21.14\%$  & $3.28$ & $1.57\%$\\
        XGB	        & 0.357 & $52.04$  & $16.45\%$  &$3.95$  & $4.41\%$\\
        GBDT	    & 0.357 & $57.37$  & $15.26\%$  &$3.68$  & $3.75\%$\\
        pure GP     & 0.357 & $77.18$  & $27.40\%$  &$5.19$  & $4.79\%$\\
        PRGP        & 0.357 & $45.83$  & $12.43\%$  &$5.09$  & $4.60\%$\\
    \midrule 
        SVM         & 0.178 & $97.02$  & $32.89\%$  & $4.55$ &$4.37\%$\\
        RF          & 0.178 & $66.44$  & $19.24\%$  & $4.06$ &$4.06\%$\\
        DNN          & 0.714 & $65.31$  & $20.13\%$  & $4.36$ & $3.17\%$\\
        XGB	        & 0.178 & $62.09$  & $20.54\%$  & $3.84$ &$4.31\%$\\
        GBDT	    & 0.178 & $67.65$  & $19.06\%$  & $3.97$ &$3.92\%$\\
        pure GP     & 0.178 & $125.28$ & $151.10\%$ & $4.39$ &$5.75\%$\\
        PRGP        & 0.178 & $45.87$  & $13.02\%$  & $4.70$ &$5.48\%$\\
    \bottomrule
    \end{tabular}
    \label{tab:benchmark5}
\end{table}

\subsection{Robustness Study}
In practice, besides the missing data, TSE also suffers from the issues of biased data. The biased data refers to that a part of data is unevenly mis-measured due to the dysfunction of the detectors. The METANET and METANET-EKF methods are not designed for dealing with either missing or biased data. In comparison to them, the proposed PRGP is capable to combine the GP and the METANET model to deal with these two challenging issues.

More specifically, Table~\ref{tab:benchmark2} summarizes their estimation performance on the biased training data.
The results show that the pure ML models such as SVM, RF, DNN, XGB, GBDT, and GP have limited resistance to high biased data, e.g., caused by traffic detector malfunctions.
The PRGP model can greatly outperform those ML models with much smaller RMSE and MAPE in flow estimations. Hence, it can be concluded that the proposed PRML model are much more robust than the pure ML models when the input data is subject to unobserved random noise. This is due to PRML's capability of adopting physics knowledge to regularized the ML training process. 
Fig.~\ref{fig:comp_noise} compares the flow and speed estimation and the ground truth for the Case I. 

\begin{table}[H]
    \caption{Comparison of model results with biased data under Case I}
    \centering
    \begin{tabular}{p{3cm}|p{2cm}|p{2cm}|p{2cm}|p{2cm}}
    \toprule
        Method &  RMSE of flow (veh/5min) & MAPE of flow & RMSE of speed (mph) & MAPE of speed \\
    \midrule
        METNET         & $ 125.08$  & $ 72.95\%$  &$5.14$ &$5.28\%$ \\
        METANET-EKF    & $ 104.91$  & $ 63.79\%$  &$4.08$ &$4.15\%$ \\
        SVM            & $ 102.15 $ & $ 43.88\%$  &$5.85$ &$6.32\%$ \\
        RF             & $ 92.91$   & $ 35.48\%$  &$3.31$ &$3.30\%$ \\
        DNN             & $ 93.76$  & $ 36.98\%$  &$3.54$ &$3.37\%$ \\
        XGB	           & $ 91.24$   & $ 32.53\%$  &$2.73$ &$3.15\%$ \\
        GBDT	       & $ 98.70$	& $ 38.87\%$  &$3.29$ &$3.26\%$ \\
        pure GP        & $ 95.32$  & $ 66.18\%$  &$4.43$ &$5.11\%$ \\
        PRGP   & $45.66$    & $14.60\%$   &$4.12$ &$3.72\%$ \\
    \bottomrule
    \end{tabular}
    \label{tab:benchmark2}
\end{table}

\begin{figure}[H]
    \centering
        \begin{subfigure}[b]{0.45\textwidth}
        \centering
        \includegraphics[width=\textwidth]{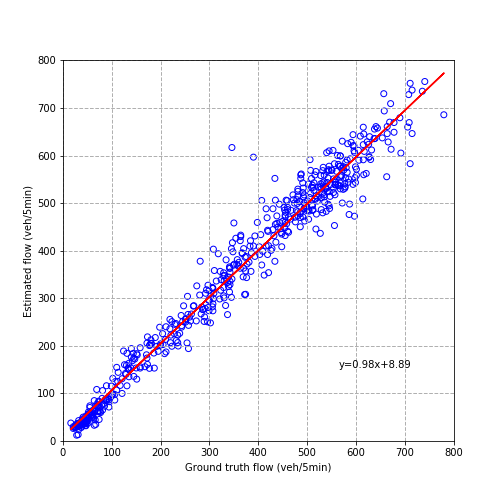}
        \caption{GP flow}
        \label{fig:gp_flow_noise}
    \end{subfigure}
    \hfill
    \begin{subfigure}[b]{0.45\textwidth}
        \includegraphics[width=\textwidth]{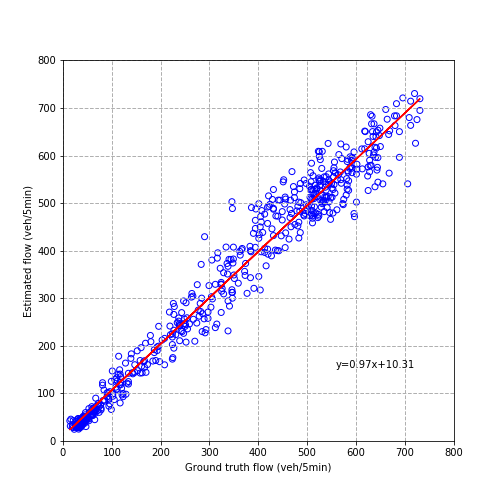}
        \caption{PRGP flow}
        \label{fig:prgp_flow_noise}
    \end{subfigure}
\end{figure}
\begin{figure}[H]\ContinuedFloat
    \centering
    \begin{subfigure}[b]{0.45\textwidth}
        \includegraphics[width=\textwidth]{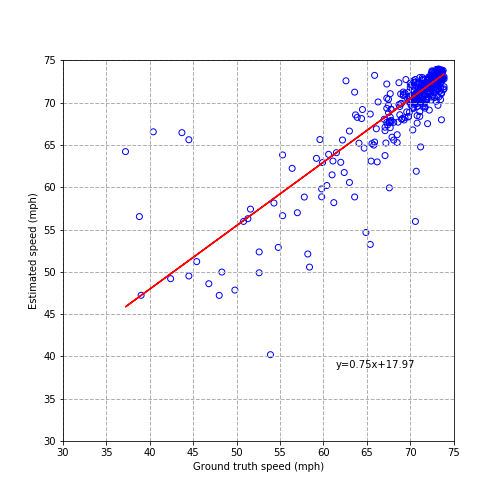}
        \caption{GP speed}
        \label{fig:gp_speed_noise}
    \end{subfigure}
    \hfill
    \begin{subfigure}[b]{0.45\textwidth}
        \includegraphics[width=\textwidth]{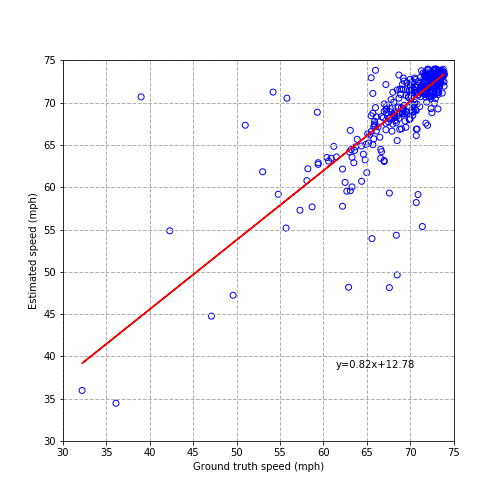}
        \caption{PRGP speed}
        \label{fig:prgp_speed_noise}
    \end{subfigure}
    \caption{Comparison between flow estimations and the ground truth with biased data under Case I}
       \vspace{-0.2in}
    \label{fig:comp_noise}
\end{figure}

\subsection{Scarceness Study}
To examine how missing data and biased data can jointly affect the models' performances, Table~\ref{tab:benchmark6} further shows the resulting sensitivity analysis of the sample ratios on the biased training dataset. From the model testing results, it can be observe that 1) the RMSE and MAPE of flow/speed estimates would be increased with the reducing of sample ratio; 2) the proposed PRGP can still yield acceptable estimation when the sample ratio is relatively large (e.g., 0.714); and 3) the MAPE of flow estimation by the proposed PRGP can go up to 52.4\% when the sample ratio is small (e.g., 0.178). Therefore, it can be concluded that the proposed PRGP can work well when the training dataset is either small or biased (but with sufficient data). However, when the training data is small and biased, none of those models tested in this study could yield satisfactory estimates.

\begin{table}[H]
    \caption{Comparison of model results with various sample ratios (biased data) under Case I}
    \centering
    
    \begin{tabular}{p{3cm}|p{2cm}|p{2cm}|p{2cm}|p{2cm}|p{2cm}}
    \toprule
        Method & Sample ratio & RMSE of flow (veh/5min) & MAPE of flow & RMSE of speed (mph) & MAPE of speed \\
    \midrule
        SVM         & 0.714 & $ 108.02 $ & $ 63.28\%$  &$5.22 $ &$5.53\%$ \\
        RF          & 0.714 & $ 88.81  $ & $ 55.52\%$  &$4.24 $ &$5.24\%$ \\
        DNN          & 0.714 & $ 90.48$  & $ 61.21\%$  &$4.53 $ &$4.67\%$ \\
        XGB	        & 0.714 & $ 92.50$   & $ 52.09\%$  &$4.94 $ &$5.91\%$ \\
        GBDT	    & 0.714 & $ 87.98$	 & $ 52.80\%$  &$4.27 $ &$5.23\%$ \\
        pure GP     & 0.714 & $ 92.30$    & $ 88.8\%$  &$ 4.36$ &$4.29\%$\\
        PRGP        & 0.714 & $ 50.21$    & $ 18.31\%$ &$ 4.30$ &$4.21\%$\\
    \midrule
        SVM         & 0.357 & $ 112.94 $ & $ 61.59\%$  &$5.38 $ &$5.72\%$ \\
        RF          & 0.357 & $ 95.55$   & $ 51.77\%$  &$4.63 $ &$5.53\%$ \\
        DNN          & 0.357 & $ 95.21$  & $ 54.12\%$  &$4.71 $ &$6.12\%$ \\
        XGB	        & 0.357 & $ 98.09$   & $ 53.40\%$  &$5.20 $ &$6.17\%$ \\
        GBDT	    & 0.357 & $ 95.40$	 & $ 52.21\%$  &$4.62 $ &$5.54\%$ \\
        pure GP     & 0.357 & $127.24$   & $ 86.50\%$  &$5.67$& $5.30\%$ \\
        PRGP        & 0.357 & $65.66$    & $ 34.60\%$  &$5.12$& $5.66\%$ \\
    \midrule
        SVM         & 0.178 & $111.26$ & $ 58.92\%$  &$5.32 $    &$5.78\%$ \\
        RF          & 0.178 & $98.70$  & $ 56.38\%$  &$5.09 $    &$5.95\%$ \\
        DNN          & 0.178 & $98.04$ & $ 55.75\%$  &$4.87 $    &$4.93\%$ \\
        XGB	        & 0.178 & $98.61$  & $ 56.97\%$  &$5.23 $    &$5.87\%$ \\
        GBDT	    & 0.178 & $97.03$  & $ 56.61\%$  &$5.06 $    &$5.90\%$ \\
        pure GP     & 0.178 & $132.22$ & $ 88.80\%$  &$4.36 $    &$4.30\%$\\
        PRGP        & 0.178 & $71.21$  & $ 52.4\%$   &$5.33 $    &$5.88\%$\\
    \bottomrule
    \end{tabular}
    \label{tab:benchmark6}
\end{table}

\section{Conclusions and Future Research Directions}\label{sec:5}
In the literature, traffic flow models have been well developed to explain the traffic phenomena, however, have theoretical difficulties in stochastic formulations and rigorous estimation.
In view of the increasing availability of data, the data-driven methods are prevailing and fast-developing, however, have limitations of lacking sensitivity of irregular events and compromised effectiveness in sparse data.

To address the issues of both methods, an assimilation-imputation hybrid method to take the advantages of both methods is investigated.
The data imputation is handled by Gaussian Process (GP) considering the missing data and measure noises while the data assimilation is captured by the traffic models.
By hybridizing them, a Physics Regularized Gaussian Process (PRGP) model is proposed to encode the physics knowledge into GP, such as discretized traffic flow models, in the Bayesian inference structure.
The physics models is encoded as the GP to regularize the conventional constraint-free Gaussian process as a soft constraint.
To estimate the proposed PRGP, a posterior regularized inference algorithm is derived and implemented.
A preliminary real-world case study is conducted on PeMS detection data collected from a freeway segment in Utah and the influential discrete traffic flow models and estimation methods are tested.
In comparison to the pure ML methods and the traffic flow models, the numerical results justify the effectiveness and the robustness of the proposed method.
In comparison to the traffic flow models, those ML models show better performance under the scenario of undetected locations. When the training data is accurate and sufficient, the proposed PRGP methods show similar performance as the pure GP. However, when dealing with biased dataset, the proposed PRGP show superior accuracy.

Please note that this study only offer a modeling method of encoding physics traffic flow models into GP. However, the similar concept may be applicable to other base ML models such as Neural Networks, Random Forest, etc. Due to the different model assumptions and architectures, more investigations would be needed in the future work.

\section*{Acknowledgement}
This research is supported by the National Science Foundation grant "\# 2047268 CAREER: Physics Regularized Machine Learning Theory: Modeling Stochastic Traffic Flow Patterns for Smart Mobility Systems".

\bibliography{main}
\end{document}